\documentclass[12pt]{article}

\usepackage{ntheorem}
\usepackage{amsmath}
\usepackage{latexsym}
\usepackage{xspace}
\usepackage{amssymb}
\usepackage{fancybox}
\usepackage{graphicx}
\usepackage{hyperref}
\usepackage{algorithmic}
\usepackage{algorithm}
\usepackage{fancyhdr}
\usepackage{pdfpages}
\newlength\FHoffset
\fancyheadoffset{\FHoffset}

\newlength\FHleft
\newlength\FHright

\newbox\FHline
\setbox\FHline=\hbox{\hsize=\paperwidth%
  \hspace*{\FHleft}%
  \rule{\dimexpr\headwidth-\FHleft-\FHright\relax}{\headrulewidth}\hspace*{\FHright}%
}
\renewcommand\headrule{\vskip-.7\baselineskip\copy\FHline}

\makeatletter
\floatstyle{ruled}
\newfloat{fragment}{H}{lop}
\floatname{fragment}{Algorithm}
\renewcommand{\floatc@ruled}[2]{\vspace{2pt}{\@fs@cfont #1.\:} #2 \par
 \vspace{1pt}}
\makeatother

\theoremsymbol{\ensuremath{_\Box}}
\theoremstyle{break}
\theoremheaderfont{\scshape}
\theorembodyfont{\slshape}
\newtheorem{Thm}{Theorem}
\newtheorem{theorem}[Thm]{Theorem}

\newtheorem{lemma}[Thm]{Lemma}

\newtheorem{Prop}[Thm]{Proposition}

\theoremstyle{plain}
\theorembodyfont{\rmfamily}

\newtheorem{Rem}{Remark}

\newtheorem{Property}{Property}

\theorembodyfont{\itshape}
\newtheorem{Def}{Definition}
\newenvironment{proof}{\noindent {\sc Proof:}}{$\Box$ 
        }
 {
        \begin{enumerate}}{\end{enumerate}}

\newcommand{\marginlabel}[1]%
{\mbox{}\marginpar{\it{\raggedleft\hspace{0pt}#1}}}
\newcommand\card[1]{\left| #1 \right|} 
\newcommand\set[1]{\left\{#1\right\}} 

\newcommand{\iprod}[2]{\langle #1, #2 \rangle}
\newcommand{\paren}[1]{\left( #1 \right)}
\newcommand{\good}{\textsf{good}}
\newcommand{\bad}{\textsf{bad}}

\newcommand\cD{\mathcal D}

\newcommand\cP{\mathcal P}

\newcommand\cS{\mathcal S}

\newcommand{\anc}[1]{\mathsf{Anc}^{(#1)}}
\newcommand{\ancl}{B^{(\ell)}}
\newcommand{\anclp}{B^{(\ell)}_+}

\newcommand{\ancPO}[1]{B^{(#1)}_+}

\newcommand{\cPmul}{\cP_{mul}} 
\newcommand{\cPmulPlus}{\cP_{mul+}}
\newcommand{\cPsing}{\cP_{sing}}
\newcommand{\cPsingPlus}{\cP_{sing+}}

\newcommand{\sgn}{\operatorname{sgn}}
\newcommand{\supp}{\operatorname{supp}}

\newcommand{\ffigureh}[4]{\begin{figure}[!h] 
\begin{center}  
\includegraphics*[height=#2]{#1}
\end{center} 
  \caption{#3}\label{#4} 
  \end{figure}}

\newcommand{\Exp}{\mathop{\mathbb E}\displaylimits}
\newcommand{\Var}{\mathop{\mathbb{V}}\displaylimits}
\newcommand\UF{\mathit{UF}}

\newcommand{\One}{\mathbf{1}}
\def\shownotes{0}  
\ifnum\shownotes=1
\newcommand{\authnote}[2]{{ $\ll$\textsf{\footnotesize #1 notes: #2}$\gg$}}
\else
\newcommand{\authnote}[2]{}
\fi

\newcommand{\Tnote}[1]{{\authnote{Tengyu}{#1}}}
\newcommand{\Snote}[1]{{\authnote{Sanjeev}{#1}}}
\newcommand{\Rnote}[1]{{\authnote{Rong}{#1}}}

\title{Provable Bounds for Learning Some Deep Representations}

\author{Sanjeev Arora\thanks{Princeton University, Computer Science Department and Center for Computational Intractability. Email: arora@cs.princeton.edu. This work is supported by the NSF grants CCF-0832797, CCF-1117309, CCF-1302518, DMS-1317308, and Simons Investigator Grant.} \and Aditya Bhaskara  \thanks{Google Research NYC.  Email: bhaskara@cs.princeton.edu.  The work was done while the author was a Postdoc at EPFL, Switzerland.}  \and Rong Ge\thanks{Microsoft Research, New England. Email: rongge@microsoft.com. Part of this work was done while the author was a graduate student at Princeton University and was supported in part by NSF grants CCF-0832797, CCF-1117309, CCF-1302518, DMS-1317308, and Simons Investigator Grant.} \and Tengyu Ma\thanks{Princeton University, Computer Science Department and Center for Computational Intractability. Email: tengyu@cs.princeton.edu. This work is supported by the NSF grants CCF-0832797, CCF-1117309, CCF-1302518, DMS-1317308, and Simons Investigator Grant.}}

\begin{document}
\maketitle
\date{}




\begin{abstract}

We give  algorithms with provable guarantees that learn a class of deep nets in the generative model view popularized by Hinton and others. 
Our generative model is an $n$ node multilayer neural net that has degree at most $n^{\gamma}$ for some $\gamma <1$ and each edge has a random edge weight in 
$[-1,1]$. Our algorithm learns {\em almost all} networks in this class with polynomial running time. The sample complexity is quadratic or cubic depending upon the details of the model.

The algorithm uses layerwise learning. It is based upon a novel idea of observing correlations among features and using these to infer the underlying edge structure via a global graph recovery procedure. The analysis  of the algorithm reveals interesting structure of  neural networks with random edge weights. \footnote{The first 18 pages of this document serve as an extended abstract of the paper, and a long technical appendix follows. }

\end{abstract}
\section{Introduction}

Can we provide theoretical explanation for the practical success of deep nets?
Like many other ML tasks, learning deep neural nets is NP-hard, 
and in fact seems \textquotedblleft badly NP-hard\textquotedblright  because of many layers of hidden variables connected by nonlinear operations.
Usually one imagines that NP-hardness is not a barrier to provable algorithms in ML because the inputs to the learner are drawn from some simple distribution and are not worst-case. 
 This hope was recently borne out	
in case of generative models such as HMMs,
Gaussian Mixtures, LDA etc., for which learning algorithms with provable guarantees were given~\cite{HKZ12, MV:many_gaussians, HK12, TopicModels, SpectralLDA}. However, 
supervised learning of neural nets even on random inputs 
still seems as hard as cracking cryptographic schemes: this holds for depth-$5$ neural nets~\cite{jackson2002learnability} and  even ANDs of thresholds (a simple depth two network)~\cite{klivans2009cryptographic}. 


However, modern deep nets are not \textquotedblleft just\textquotedblright neural nets (see the survey~\cite{BengioSurvey}).
The underlying assumption is that the  net (or some modification) 
can be run {\em in reverse} to get a {\em generative model} for a distribution 
that  is a close fit to the empirical input distribution. Hinton promoted this viewpoint,
and suggested modeling each level as a Restricted Boltzmann Machine (RBM), which is \textquotedblleft reversible\textquotedblright in this sense. 
Vincent et al.~\cite{DBLP:conf/icml/VincentLBM08} suggested using many layers of a {\em denoising autoencoder}, a generalization of the RBM 
that consists of a pair of encoder-decoder functions
(see Definition~\ref{def:autoencode}). 
These viewpoints allow a different learning methodology than classical backpropagation: {\em layerwise} learning of the net, and in fact {\em unsupervised} learning.
The bottom (observed) layer is learnt in {\em unsupervised fashion} using the provided data. This gives values for the next layer of hidden variables, which are used as the “data” to learn the next higher layer, and so on. 
The final net thus learnt is also a good generative model for the distribution of the bottom layer. In practice
the unsupervised phase is followed by supervised training\footnote{Recent work suggests that
classical backpropagation-based learning of neural nets together with a few modern ideas like convolution and dropout training also performs
very well~\cite{ImageNet}, though the authors suggest that unsupervised pretraining  should help further.}.


This viewpoint of reversible deep nets is more promising for theoretical work
because it involves a generative model, and also seems to get around cryptographic hardness.
But many barriers still remain. There is no known mathematical condition that describes neural nets that are or are not denoising autoencoders. 
 Furthermore, learning even a 
 a {\em single} layer sparse denoising autoencoder seems at least as hard as learning sparse-used
{\em overcomplete dictionaries} (i.e., a single hidden layer with linear operations),  for which there were no provable bounds at all until the very recent manuscript~\cite{AGM}\footnote{The parameter choices 
in that manuscript make it less interesting in context of deep learning, since
the hidden layer is required to have no more than $\sqrt{n}$ nonzeros where $n$ is the size
of the observed layer ---in other words, the observed vector must be highly compressible.}.

The current paper presents both an interesting family of denoising autoencoders as well as
new algorithms to provably learn almost all 
models in this family. 
%
 Our ground truth generative model is a simple multilayer neural net with edge weights in
$[-1,1]$ and simple threshold (i.e., $> 0$) computation at the nodes.  
A $k$-sparse $0/1$ assignment is provided at the top hidden layer,  which is computed upon by 
successive hidden layers in the obvious way until the \textquotedblleft observed vector\textquotedblright  appears
at the bottommost layer. 
If one makes no further assumptions, then the problem of learning the network given samples from the bottom layer
is still harder than breaking some cryptographic schemes. (To rephrase this in autoencoder terminology: our model comes equipped with a {\em decoder} function at each layer.
But this is not enough to  guarantee  an efficient  {\em encoder}  function---this may be tantamount to breaking cryptographic schemes.)


So we make the following additional assumptions about the unknown \textquotedblleft ground truth deep net\textquotedblright (see Section~\ref{sec:prelim}):
(i) Each feature/node activates/inhibits at most $n^{\gamma}$ features at the layer below, and is itself activated/inhibited by at 
most $n^{\gamma}$ features in the layer above, where $\gamma$ is some small constant; in other words the ground truth net is not a complete graph.
(ii) The graph of these edges is chosen
at random and the weights on these edges are random numbers in $[-1,1]$. 


Our algorithm learns {\em almost all} networks in this class very efficiently and with low sample complexity; see Theorem~\ref{thm:icmlmain}. 
The algorithm outputs a network whose generative behavior is statistically indistinguishable from the
ground truth net. (If the weights are discrete, say in $\set{-1,1}$ then it exactly learns the ground truth net.)

Along the way we exhibit interesting properties of such randomly-generated neural nets.  (a) Each pair of adjacent layers constitutes
a denoising autoencoder in the sense of Vincent et al.; see Lemma~\ref{lem:autoencode}. Since the model definition already includes a decoder, this
involves showing the {\em existence} of an encoder that completes it into an autoencoder. 
(b) The encoder is actually the same neural network run 
in reverse by appropriately changing the thresholds at the computation nodes. (c) The reverse computation is stable to dropouts and noise. (d) The distribution generated by a two-layer net cannot 
be represented by {\em any} single layer neural net (see Section~\ref{sec:lowerbounds}), which in turn 
suggests that a random t-layer network cannot be represented by {\em any} $t/2$-level 
neural net\footnote{Formally proving this  for $t>3$ is difficult however since
showing limitations of even 2-layer neural nets is a major open problem in computational complexity theory. Some deep learning papers mistakenly cite an old paper for such a result, but the result that actually exists is far weaker.}.

Note that properties (a) to (d) are {\em assumed} in modern deep net work: for example (b) is a heuristic trick called
\textquotedblleft weight tying\textquotedblright. The fact that they {\em provably} hold for our
random generative model can be seen as some theoretical validation of those assumptions.

\noindent{\bf Context.}  
Recent papers have given theoretical analyses of models with multiple levels of hidden features, including
SVMs~\cite{KernelDeep, OhadDeep}.
However, none of these solves the task  of recovering a ground-truth neural network given its output distribution.

 Though real-life neural nets are not random, our consideration of  random deep networks makes some sense
 for theory. Sparse denoising autoencoders are reminiscent of other objects such as error-correcting codes,
 compressed sensing, etc. which were all first analysed in the random case. 
 As mentioned, provable reconstruction of the hidden layer (i.e., input encoding)  in a {\em known} autoencoder 
 already seems a nonlinear generalization of compressed sensing, whereas even the usual (linear) version of compressed sensing seems 
possible only if the adjacency matrix has 
  ``random-like" properties (low coherence or restricted isoperimetry or lossless expansion). 
 In fact our result that a single layer of our generative model is a sparse denoising autoencoder can be seen as an analog of the fact that random matrices are good for compressed sensing/sparse reconstruction (see Donoho~\cite{donoho2006compressed} for general matrices and  Berinde et al.~\cite{IndykStrauss} for 
 sparse matrices).
 %
  Of course, in compressed sensing the matrix of edge weights is known whereas here it has to be learnt, 
  which is the main contribution of our work.
  Furthermore, we show that our algorithm for learning a single layer of weights can be extended to do layerwise learning of the entire network.
  
  Does our algorithm yield new approaches in practice? We discuss this possibility after sketching our algorithm in the next section.

\section{Definitions and Results}
\label{sec:prelim}

Our generative neural net model (\textquotedblleft ground truth\textquotedblright) has $\ell$ hidden layers of vectors of binary variables $h^{(\ell)}$, $h^{(\ell-1)}$, .., $h^{(1)}$ (where $h^{(\ell)}$ is the top layer) and an observed layer $y$
at bottom. The number of vertices at layer $i$ is denoted by $n_i$, and the set of edges between layers $i$ and $i+1$ by $E_i$.
In this abstract we assume all $n_i =n$; in appendix we allow them to differ.\footnote{When the layer sizes differ the
sparsity of the layers are related by $\rho_{i+1} d_{i+1} n_{i+1}/2 = \rho_{i}n_i$. Nothing much else changes.} (The long technical appendix serves partially as a full version of the paper with exact parameters and complete proofs). 
The weighted graph between layers $h^{(i)}$ and $h^{(i+1)}$ has degree at most $d = n^\gamma$ and all edge weights are in $[-1,1]$.
The generative model works like a neural net where the threshold at every node\footnote{It is possible to allow  these thresholds to be higher and to vary between the nodes,
but the calculations are harder and the algorithm is much less efficient.} is $0$. The top layer $h^{(\ell)}$ is initialized with a $0/1$ assignment where the set of nodes that are $1$ is picked uniformly\footnote{It is possible to prove the result when the top layer has not a random sparse vector and has some bounded correlations among them. This makes the algorithm more complicated.} among all sets of size $\rho_{\ell} n$. For 
$i=\ell$ down to $2$, each node in layer $i -1$ computes a weighted sum of its neighbors in layer $i$, and becomes $1$ iff that sum strictly exceeds $0$. We will use $\sgn(x)$ to denote the threshold function that is $1$ if $x >0$ and $0$ else.
(Applying $\sgn()$ to a vector involves applying it componentwise.) Thus the network computes as follows:  $h^{(i-1)} = \sgn(G_{i-1}h^{(i)})$ for all $i >0$ and
  $h^{(0)} = G_{0}h^{(1)}$ (i.e., no threshold at the observed layer)\footnote{
  It is possible to stay with a generative deep model in which the last layer also has $0/1$ values.
  Then our calculations require the fraction of $1$'s in the lowermost (observed) layer to be at most $1/\log n$. This could be an OK model if one assumes that some handcoded net (or a nonrandom layer like 
  convolutional net) has been used on the real data to produce a sparse encoding, which is the bottom layer of our generative model.
  
  However, if one desires a generative model in which the observed layer is not sparse, then we can do this by allowing real-valued
  assignments at the observed layer, and remove the threshold gates there. This is the model described here.}. Here  $G_i$ stands for both the weighted bipartite graph at a level 
  and its weight matrix.



\ffigureh{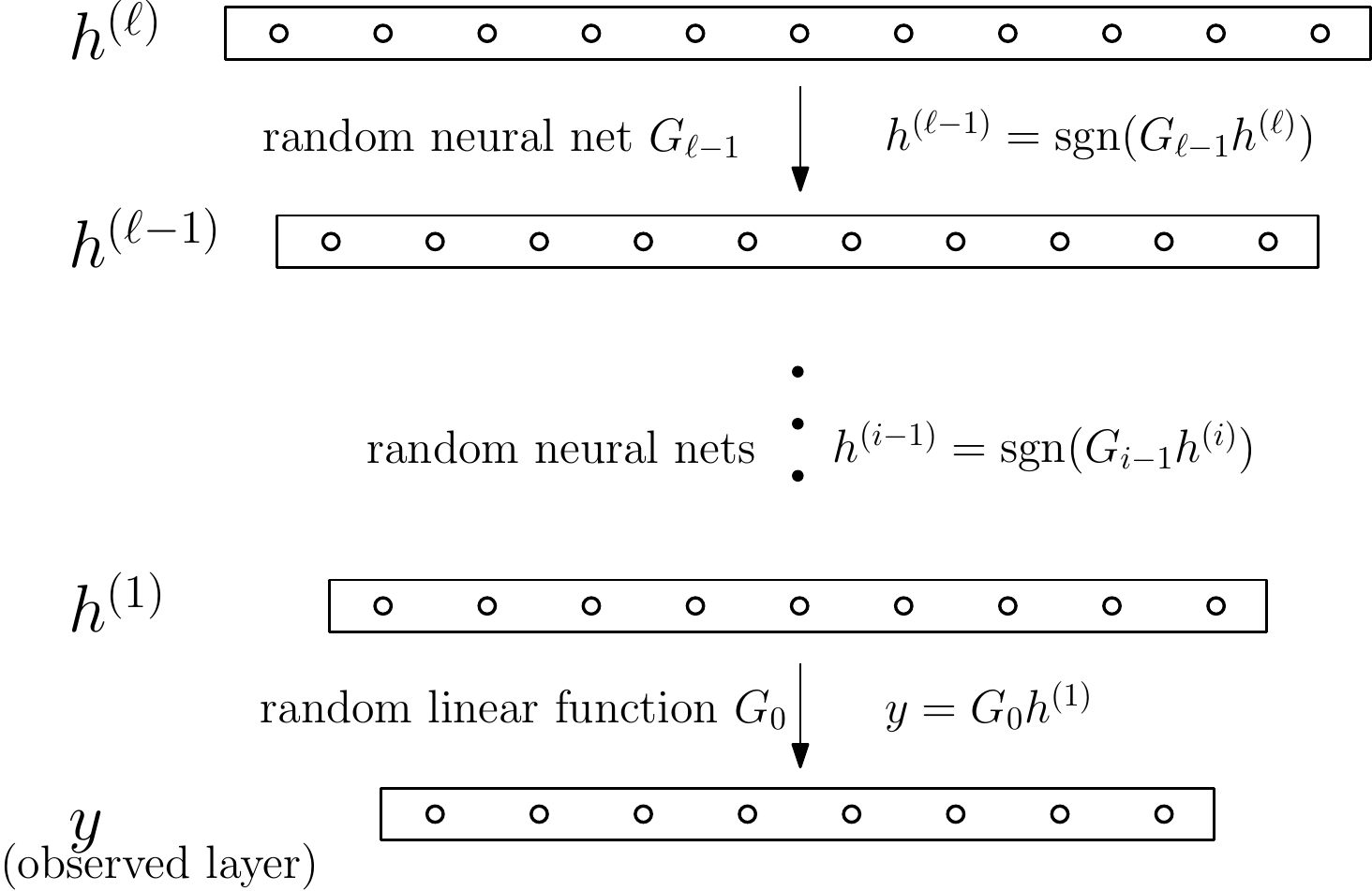}{1.7in}{Example of a deep network}{fig2}

{\em Random deep net assumption:} We assume that in this ground truth the edges between layers are chosen randomly subject to expected degree $d$ being\footnote{In the appendix we allow degrees to be different for different layers.} $n^\gamma$, where
$\gamma < 1/(\ell +1)$, and each edge $e \in E_i$ carries a weight  that is chosen randomly in $[-1,1]$. This is our model
$\mathcal{R}(\ell, \rho_l, \{G_i\})$. We also consider ---because it leads to a simpler and more efficient learner---a model where edge weights are random in $\{\pm 1\}$ instead of $[-1,1]$; this is called $\mathcal{D}(\ell, \rho_{\ell}, \{G_i\})$. Recall that $\rho_{\ell} >0$ is such that the $0/1$ vector input
at the top layer has $1$'s in a random subset of $\rho_{\ell} n$ nodes.

It can be seen that since the network is random of degree $d$, applying a $\rho_{\ell} n$-sparse vector at the top layer is likely to produce the following density of $1$'s (approximately) at the successive layers: $\rho_{\ell} d/2, \rho_{\ell} (d/2)^2$, etc..
We assume the density of last layer $\rho_{\ell} d^\ell/2^\ell = O(1)$. This way the density at the last-but-one layer is $o(1)$, and the last layer is real-valued and dense. 


Now we state our main result. Note that $1/\rho_{\ell}$ is at most $n$. 

\begin{Thm} \label{thm:icmlmain}
When degree $d = n^\gamma$ for $0 < \gamma \le 0.2$, density $\rho_\ell (d/2)^l = C$ for some large constant $C$\footnote{In this case the output is dense}, the network model $\mathcal{D}(\ell, \rho_{\ell}, \{G_i\})$ can be learnt using $O(\log n/\rho_{\ell}^2)$
samples and $O( n^2 \ell)$ time. The network model $\mathcal{R}(\ell, \rho_{\ell}, \{G_i\})$ can be learnt in polynomial time and using $O(n^3 \ell^2\log n/\eta^2)$ samples, where $\eta$ is the statistical distance
between the true distribution and that generated by the learnt model.
\end{Thm}

\vspace*{-0.15in}
\noindent{\bf Algorithmic ideas.} We are unable to analyse existing algorithms. 
Instead, we give new learning algorithms that exploit the very same structure that makes these random 
networks interesting  in the first place i.e., each layer is a denoising autoencoder. 
The crux of the algorithm
is a new twist on 
the old Hebbian rule~\cite{Hebb1949} that ``Things that fire together wire together." In the setting of layerwise learning, this is adapted as follows: ``Nodes in the same layer that fire together a lot are likely to be connected (with positive weight) to the same node at the higher layer."  The algorithm consists of looking for such pairwise (or $3$-wise) correlations and putting together this information globally.
The global procedure boils down to the graph-theoretic problem of reconstructing a bipartite graph given pairs of nodes that are at distance $2$ in it (see Section~\ref{sec:graph_recovery}).
 This is a variant of the GRAPH SQUARE ROOT problem which is NP-complete on worst-case instances but solvable for sparse random (or random-like) graphs. 

Note that current algorithms (to the extent that they are Hebbian, roughly speaking) can also be seen as leveraging correlations. But putting together this information is done via the language of 
nonlinear optimization (i.e., an objective function with suitable penalty terms). 
Our ground truth network is indeed a particular local optimum in any reasonable formulation. It would be interesting to show that
existing algorithms provably find the ground truth in polynomial time but currently this seems difficult. 

Can our new ideas be useful in practice? We think that using a global reconstruction procedure to leverage local correlations
seems  promising, especially if it avoids the usual nonlinear optimization. Our proof currently needs that
the hidden layers are sparse, and the edge structure of the ground truth network is \textquotedblleft random like\textquotedblright (in the sense that two distinct features at a level tend to affect fairly disjoint-ish sets of
features at the next level). Finally, we note that random neural nets do seem useful in so-called {\em reservoir} computing, so perhaps
they do provide useful representational power on real data. 
Such empirical study  is left for future work.


Throughout, we need well-known properties of random graphs with expected degree $d$, such as the fact that they are expanders; these properties appear in the appendix. 
The most important one, unique neighbors property, appears in the next Section.
\section{Each layer is a Denoising Auto-encoder}\label{sec:auto_encoder}

As mentioned earlier, modern deep nets research often assumes that the net (or at least some layers in it) should approximately preserve information, and even allows easy going back/forth between representations in two adjacent layers
(what we earlier called \textquotedblleft reversibility\textquotedblright).
Below, $y$ denotes the lower layer and $h$ the higher (hidden) layer.
Popular choices of $s$ include logistic function, soft max, etc.; we use simple threshold function in our model.
\begin{Def} {\sc (Denoising autoencoder)} \label{def:autoencode}
An {\em autoencoder} consists of a {\em decoding function} $D(h)= s(Wh +b)$ and an {\em encoding function} $E(y) = s(W'y + b')$ where $W, W'$ are linear transformations, $b,b'$ are fixed vectors and $s$ is a nonlinear 
function that acts identically on each coordinate. The autoencoder is {\em denoising}  if $E(D(h) + \eta) = h$ with high probability where $h$ is drawn from the distribution of the 
hidden layer, $\eta$ is a noise vector drawn from the noise distribution, and $D(h) +\eta$ is a shorthand for ``$D(h)$ corrupted with noise $\eta$."
The autoencoder is said to use {\em weight tying} if $W' = W^T$. 
\end{Def} 


 In empirical work the denoising autoencoder property is only implicitly imposed on the deep net 
 by minimizing the reconstruction error $||y - D(E(y +\eta))||$, where $\eta$ is the noise vector. Our definition is slightly different but is actually stronger since
 $y$ is exactly $D(h)$ according to the generative model. Our definition implies the existence of an encoder $E$ that makes the penalty term exactly zero.
We show that in our ground truth net
(whether from model $\mathcal{D}(\ell, \rho_{\ell}, \{G_i\})$ or $\mathcal{R}(\ell, \rho_{\ell}, \{G_i\})$) every pair of successive levels whp satisfies this definition,  
and with weight-tying.

We show a single-layer random network is a denoising autoencoder if the input layer is a random $\rho n$ sparse vector, and the output layer has density $\rho d/2 < 1/20$.

\begin{lemma}\label{lem:autoencode}
If $\rho d < 0.1 $ (i.e., the assignment to the observed layer is also fairly sparse) then the single-layer network above is a denoising autoencoder with high probability (over the choice of the random graph and weights), where the noise distribution 
is allowed to flip every output bit independently with probability $0.1$. It uses weight tying.
\end{lemma}

The proof of this lemma highly relies on a property of random graph, called the {\em strong unique-neighbor property.}

For any node $u\in U$ and any subset $S\subset U$ , let $\UF(u,S)$ be the sets of unique neighbors of $u$ with respect to $S$, 
\[\UF(u,S) \triangleq \{v\in V: v\in F(u), v\not\in F(S\setminus \{u\})\}\]
\begin{Property}\label{prop:unique_neighbor}
In a bipartite graph $G(U,V,E,w)$, a node $u\in U$ has $(1-\epsilon)$-unique neighbor property with respect to $S$ if
\begin{equation}
\sum_{v\in \UF(u,S)} |w(u,v)|  \ge (1-\epsilon)\sum_{v\in F(u)} |w(u,v)|
\end{equation}
The set $S$ has $(1-\epsilon)$-strong unique neighbor property if for every $u\in U$, $u$ has $(1-\epsilon)$-unique neighbor property with respect to $S$. 
\end{Property}

When we just assume $\rho d \ll n$, this property does not hold for all sets of size $\rho n$. However, for any fixed set $S$ of size $\rho n$, this property holds with high probability over the randomness of the graph.

Now we sketch the proof for Lemma~\ref{lem:autoencode} (details are in Appendix).For convenience 
assume the edge weights are in $\set{-1,1}$.

 First, the decoder definition is implicit in our 
generative model: $y =\sgn(Wh)$. (That is,
 $b=\vec{0}$ in the autoencoder definition.)
Let the encoder be $E(y) = \sgn(W^T y +b')$ for  $b'= 0.2d \times \vec{1}$.
In other words, the same bipartite graph and different thresholds can transform an assignment on the lower level to the one at the higher level.


To prove this consider the strong unique-neighbor property of the network. For the set of nodes that are $1$ at the higher level, 
%
%
a majority of their neighbors at the lower level are adjacent only to them and to no other nodes that are $1$. The unique neighbors with a positive edge will always be 1 because there are no $-1$ edges that can cancel the $+1$ edge (similarly the unique neighbors with negative edge will always be 0). Thus by looking at the set of nodes 
that are $1$ at the lower level, one can easily infer the correct $0/1$ assignment to the higher level by doing a simple
threshold of say $0.2 d$ at each node in the higher layer.

\section{Learning a single layer network}\label{sec:one-layer}
Our algorithm, outlined below (Algorithm~\ref{alg:highlevel}), learns the network layer by layer starting from the bottom. Thus the key step is that of learning a single layer network, which we now focus on.\footnote{Learning the bottom-most (real valued) layer is mildly different and is done in Section~\ref{sec:lastlayer}.}  This step, as we noted, amounts to learning 
{\em nonlinear }dictionaries with random dictionary elements. The algorithm 
 illustrates how we leverage the sparsity and the randomness of the {\em support graph}, and use pairwise or 3-wise correlations combined with our graph recovery procedure of Section~\ref{sec:graph_recovery}. 
 We first give a simple algorithm and then outline one that works with better parameters.

\begin{algorithm}[!h]
\caption{High Level Algorithm} \label{alg:highlevel}
\begin{algorithmic}[1]
\REQUIRE{samples $y$'s generated by a deep network described in Section~\ref{sec:prelim}}
\ENSURE{the network/encoder and decoder functions}
\FOR{$i = 1$ TO $l$}
\STATE Construct {\em correlation graph} using samples of $h^{(i-1)}$
\STATE Call RecoverGraph to learn the positive edges $E_i^+$
\STATE Use PartialEncoder to encode all $h^{(i-1)}$ to $h^{(i)}$
\STATE Use LearnGraph/LearnDecoder to learn the graph/decoder between layer $i$ and $i-1$.
\ENDFOR
\end{algorithmic}
\end{algorithm}

For simplicity we describe the algorithm when edge weights are $\set{-1, 1}$, and sketch the differences for 
real-valued weights at the end of this section.

 The hidden layer and observed layer each have $n$ nodes, and the generative model assumes the assignment to the hidden layer is  a random $0/1$ assignment with $\rho n$ 
nonzeros. 

Say two nodes in the observed layer are {\em related} if they have a common neighbor in the hidden layer to which they are attached via a $+1$ edge.

\noindent{\sc Step 1:} {\em Construct correlation graph:}
This step is a new twist on the classical Hebbian rule (\textquotedblleft things that fire together wire together\textquotedblright).

\begin{algorithm}[!h]
 \caption{PairwiseGraph}
 \label{alg:pairwise-corr}
\label{alg:pairwisegraph}
\begin{algorithmic}
\REQUIRE{$N = O(\log n/\rho)$ samples of $y = \sgn(Gh)$, }
\ENSURE{$\hat{G}$ on vertices $V$, $u,v$ connected if related}
\FOR{each $u,v$ in the output layer}
	\IF{$\ge \rho N/3$ samples have $y_u=y_v = 1$}
		\STATE connect $u$ and $v$ in $\hat{G}$
	\ENDIF
\ENDFOR
\end{algorithmic}
\end{algorithm}

\noindent{\bf Claim} {\em In a random sample of the output layer, related pairs $u, v$ are both $1$ with probability at least $0.9\rho$, while
unrelated pairs are both $1$ with probability at most $(\rho d)^2$. }

\noindent {\em (Proof Sketch):}
First consider a related pair $u,v$, and let $z$ be a vertex with $+1$ edges to $u$, $v$. Let $S$ be the set of neighbors of $u$, $v$ excluding $z$. The size of $S$ cannot be much larger than $2d$. Under the choice of parameters, we know $\rho d\ll 1$, so the event $h_S = \vec{0}$ conditioned on $h_z = 1$ has probability at least 0.9.
Hence the probability of $u$ and $v$ being both $1$ is at least $0.9\rho$. Conversely, if $u, v$ are 
unrelated then for both $u, v$  to be $1$ there must be two different causes, namely, nodes $y$ and $z$ that are $1$, and additionally, are connected to $u$ and $v$ respectively via $+1$ edges. The chance of such $y,z$ existing in a random sparse assignment is at most $(\rho d)^2$ by union bound.

Thus, if $\rho$ satisfies $(\rho d)^2 < 0.1 \rho$, i.e., $\rho < 0.1/d^2$, then using $O(\log n/\rho^2)$ samples we can recover all related pairs whp, finishing the step.

\noindent{\sc Step 2:} {\em Use {\sc graph recover} procedure to find all edges that have weight $+1$.} (See Section~\ref{sec:graph_recovery}
for details.)

\noindent{\sc Step 3:} {\em Using the $+1$ edges to encode all the samples $y$.}

\begin{algorithm}[!h]
\caption{PartialEncoder}
\label{alg:finding_S}
\label{alg:partialencoder}
  \begin{algorithmic}
  \REQUIRE{positive edges $E^+$, $y = \sgn(Gh)$, threshold $\theta$}
  \ENSURE{the hidden variable $h$}
 \STATE Let $M$ be the indicator matrix of $E^+$ ($M_{i,j} = 1$ iff $(i,j)\in E^+$)
 \STATE {\bf return} $h = \sgn(M^T y - \theta \vec{1})$
  \end{algorithmic}
\end{algorithm}

Although we have only recovered the positive edges, we can use {\sc PartialEncoder} algorithm to get $h$ given $y$! 

\begin{lemma}
If support of $h$ satisfies $11/12$-strong unique neighbor property, and $y = \sgn(Gh)$, then Algorithm~\ref{alg:finding_S} outputs $h$ with $\theta = 0.3d$.
\end{lemma}

This uses the unique neighbor property: for every $z$ with $h_z = 1$, it has at least $0.4d$ unique neighbors that are connected with $+1$ edges. All these neighbors must be $1$ so $[(E^+)^T y]_z \ge 0.4d$. On the other hand, for any $z$ with $h_z=0$, the unique neighbor property (applied to supp$(h) \cup \{z\}$) implies that $z$ can have at most $0.2d$ positive edges to the $+1$'s in $y$. Hence $h = \sgn((E^+)^T y - 0.3d\vec{1})$.

\noindent{\sc Step 4:} {\em Recover all weight $-1$ edges.} 
\begin{algorithm}
\caption{Learning Graph}\label{alg:finding_neg_edges:single}
\begin{algorithmic}[1]
\REQUIRE positive edges $E^+$,  samples of $(h,y)$
\ENSURE $E^{-}$
\STATE $R \leftarrow (U\times V)\setminus E^+$.
\FOR{each of the samples $(h,y)$, and each $v$}
\STATE Let $S$ be the support of $h$
	\IF{$y_v = 1$ and $S \cap B^+(v) = \{u\}$ for some $u$}
		\FOR{$s\in S$}
			\STATE remove $(s,v)$ from $R$. 
		\ENDFOR
	\ENDIF 
\ENDFOR
\STATE {\bf return} $R$
\end{algorithmic}
\end{algorithm}

Now consider many pairs of $(h,y)$, where $h$ is found using Step 3. Suppose in some sample, $y_u=1$ for some $u$, and exactly one neighbor of $u$ in the $+1$ edge graph (which we know entirely) is in supp$(h)$. Then we can conclude that for any $z$ with $h_z=1$, there cannot be a $-1$ edge $(z,u)$, as this would cancel out the unique $+1$ contribution.

\begin{lemma}
Given $O(\log n/(\rho^2 d))$ samples of pairs $(h,y)$, with high probability (over the random graph and the samples) Algorithm~\ref{alg:finding_neg_edges:single} outputs the correct set $E^-$.
\end{lemma}

To prove this lemma, we just need to bound the probability of the following event for any non-edge $(x,u)$: $h_x = 1$, $\card{\supp(h) \cap B^+(u)} = 1$, $\supp(h)\cap B^-(u) = \emptyset$ ($B^+,B^-$ are positive and negative parents). These three events are almost independent, the first has probability $\rho$, second has probability $\approx \rho d$ and the third has probability almost 1.


\paragraph{Leveraging $3$-wise correlation:} The above sketch used pairwise correlations to recover the $+1$ weights when $\rho < 1/d^2$, roughly. It turns out that using $3$-wise correlations allow us to find correlations under a weaker requirement $\rho <1/d^{3/2}$. Now call three observed nodes $u, v, s$ 
{\em related} if they are connected to a common node at the hidden layer via $+1$ edges. 
Then we can prove a claim analogous to the one above, which says that for a related triple, the probability that $u,v,s$ are all $1$ is at least $0.9\rho$, while the probability for unrelated triples is roughly at most $(\rho d)^3$.  Thus as long as $\rho < 0.1/d^{3/2}$, it is possible to find related triples correctly. The {\sc graph recover algorithm} can be modified to run on $3$-uniform hypergraph consisting of these
related triples to recover the $+1$ edges.

The end result is the following theorem. This is the learner used to get the bounds stated in our main theorem.
\begin{theorem}
Suppose our generative neural net model with weights $\set{-1, 1}$ has a single layer and the assignment of the hidden layer is a random $\rho n$-sparse
vector, with $\rho\ll 1/d^{3/2}$. Then there is an algorithm that runs in $O(n(d^3+n))$ time and uses $O(\log n/\rho^2)$ samples to recover the ground truth with high probability over the randomness of the graph and the samples. 
\end{theorem}


\paragraph{When weights are real numbers.} We only sketch this and leave the details to the appendix. 
Surprisingly, steps 1, 2 and 3 still work. In the proofs, we have only used the sign of the edge weights -- the magnitude of the edge weights can be arbitrary. This is because the proofs in these steps relies on the unique neighbor property, if some node is on (has value $1$), then its unique positive neighbors at the next level will always be on, no matter how small the positive weights might be. Also notice in PartialEncoder we are only using the support of $E^+$, but not the weights.
%

After Step 3 we have turned the problem of unsupervised learning of the hidden graph to a supervised one in which the outputs are just linear classifiers over the inputs! Thus 
the weights on the edges can be learnt to any desired accuracy.

\section{Correlations in a Multilayer Network}
\label{sec:multilayer}

We now consider multi-layer networks, and show how they can be learnt layerwise using a slight modification of our one-layer algorithm at each layer. At a technical level, the difficulty in the analysis is the following: in single-layer learning, we assumed that the higher layer's assignment is a random $\rho n$-sparse binary vector. In the multilayer network, the assignments in intermediate layers (except for the top layer) do not satisfy this, but we will show that the correlations among them are low enough that we can carry forth the argument. Again for simplicity we describe the algorithm for the model $\mathcal{D}(\ell, \rho_l, \{G_i\})$, in which the edge weights are $\pm 1$. Also to keep notation simple, we describe how to bound the correlations in bottom-most layer ($h^{(1)}$). It holds almost verbatim for the higher layers. We define $\rho_i$ to be the ``expected'' number of $1$s in the layer $h^{(i)}$. Because of the unique neighbor property, we expect roughly $\rho_l (d/2)$ fraction of $h^{(\ell-1)}$ to be $1$. So also, for subsequent layers, we obtain $\rho_i = \rho_\ell \cdot (d/2)^{\ell-i}$. (We can also think of the above expression as defining $\rho_i$).

\begin{lemma}\label{lem:corrmanylayer}
Consider a network from $\mathcal{D}(\ell, \rho_l, \{G_i\})$. With high probability (over the random graphs between layers) for any two nodes $u,v$ in layer $h^{(1)}$,
\[
\Pr[h^{(1)}_u=h^{(1)}_v=1] \left\{\begin{array}{lr}
\ge \rho_2/2 & \text{if $u,v$ related} \\
\le \rho_2/4 & otherwise
\end{array}\right.
\]
\end{lemma}
\begin{proof}(outline) The first step is to show that for a vertex $u$ in level $i$, $\Pr[h^{(i)}(u)=1]$ is at least $3\rho_i/4$ and at most $5\rho_i/4$. This is shown by an inductive argument (details in the full version). (This is the step where we crucially use the randomness of the underlying graph.)

Now suppose $u, v$ have a common neighbor $z$ with $+1$ edges to both of them. Consider the event that $z$ is $1$ and none of the neighbors of $u,v$ with $-1$ weight edges are $1$ in layer $h^{(2)}$. These conditions ensure that $h^{(1)}(u) = h^{(1)}(v)=1$; further, they turn out to occur together with probability at least $\rho_2/2$, because of the bound from the first step, along with the fact that $u,v$ combined have only $2d$ neighbors (and $2d \rho_2 n \ll n$), so there is good probability of not picking neighbors with $-1$ edges.

If $u,v$ are not related, it turns out that the probability of interest is at most $2\rho_1^2$ plus a term which depends on whether $u,v$ have a common parent in layer $h^{(3)}$ in the graph restricted to $+1$ edges. Intuitively, picking one of these common parents could result in $u,v$ both being $1$. By our choice of parameters, we will have $\rho_1^2 < \rho_2/20$, and also the additional term will be $<\rho_2/10$, which implies the desired conclusion.  
\end{proof}

Then as before, we can use graph recovery to find all the $+1$ edges in the graph at the bottom most layer. This can then be used (as in Step 3) in the single layer algorithm to encode $h^{(1)}$ and obtain values for $h^{(2)}$. Now as before, we have many pairs $(h^{(2)}, h^{(1)})$, and thus using precisely the reasoning of Step 4 earlier, we can obtain the full graph at the bottom layer.

This argument can be repeated after `peeling off' the bottom layer, thus allowing us to learn layer by layer.

\section{Graph Recovery}\label{sec:graph_recovery}


Graph reconstruction consists of recovering a  graph given information about its
subgraphs~\cite{bondy1977graph}. A prototypical problem is the {\em Graph Square Root} problem, which 
calls for recovering a graph given all pairs of nodes whose distance is at most $2$. This is
NP-hard.  

\begin{Def}[Graph Recovery]\label{def:graphrecovery}
Let $G_1(U, V, E_1)$ be an unknown random bipartite graph between $|U| = n$ and $|V| = n$ vertices where each edge is picked with probability $d/n$ independently. 

\noindent {\em Given:} Graph $G(V, E)$ where $(v_1,v_2) \in E$ iff $v_1$ and $v_2$ share a common parent in $G_1$ (i.e. $\exists u\in U$ where $(u,v_1)\in E_1$ and $(u, v_2)\in E_1$).

\noindent {\em Goal:} Find the bipartite graph $G_1$.
\end{Def}

Some of our algorithms (using $3$-wise correlations) need to solve analogous problem where we are given triples of nodes which are mutually at distance $2$ from each other, 
which we will not detail for lack of space.

We let $F(S)$ (resp. $B(S)$) denote the set of neighbors of $S \subseteq U$ (resp. $\subseteq V$) in $G_1$. Also $\Gamma(\cdot)$ gives the set of neighbors in $G$. Now for the recovery algorithm to work, we need the following properties (all satisfied whp by random graph when $d^3/n \ll 1$):

\begin{enumerate}
\item For any $v_1,v_2\in V$, \newline $\card{(\Gamma(v_1)\cap \Gamma(v_2)) \backslash (F(B(v_1)\cap B(v_2)))} < d/20$.
\item For any $u_1, u_2 \in U$, $\card{F(u_1)\cup F(u_2)} > 1.5 d$.
\item For any $u\in U$, $v\in V$ and $v\not\in F(u)$, $\card{\Gamma(v)\cap F(u)} < d/20$.
\item For any $u \in U$, at least $0.1$ fraction of pairs $v_1,v_2\in F(u)$ does not have a common neighbor other than $u$.
\end{enumerate}

The first property says ``most correlations are generated by common cause'': all but possibly $d/20$ of the common neighbors of $v_1$ and $v_2$ in $G$, are in fact neighbors of a common neighbor of $v_1$ and $v_2$ in $G_1$. 

The second property basically says the sets $F(u)$'s should be almost disjoint, this is clear because the sets are chosen at random.

The third property says if a vertex $v$ is not related to the cause $u$, then it cannot have correlation with all many neighbors of $u$.

The fourth property says every cause introduces a significant number of correlations that is unique to that cause.

In fact, Properties 2-4 are closely related from the unique neighbor property.

\begin{lemma}
When graph $G_1$ satisfies Properties 1-4, Algorithm~\ref{alg:community_finding_simple} successfully recovers the graph $G_1$ in expected time $O(n^2)$.
\label{lem:graphrecoverdeterministic}
\end{lemma}

\begin{proof}
We first show that when $(v_1,v_2)$ has more than one unique common cause, then the condition in the if statement must be false. This follows from Property 2. We know the set $S$ contains $F(B(v_1)\cap B(v_2))$. If $\card{B(v_1)\cap B(v_2)} \ge 2$ then Property 2 says $\card{S} \ge 1.5d$, which implies the condition in the if statement is false.

Then we show if $(v_1,v_2)$ has a unique common cause $u$, then $S'$ will be equal to $F(u)$. By Property 1, we know $S = F(u) \cup T$ where $\card{T} \le d/20$.

For any vertex $v$ in $F(u)$, it is connected to every other vertex in $F(u)$. Therefore $\card{\Gamma(v)\cap S} \ge \card{\Gamma(v)\cap F(u)} \ge 0.8d-1$, and $v$ must be in $S'$.

For any vertex $v'$ outside $F(u)$, by Property 3 it can only be connected to $d/20$ vertices in $F(u)$. Therefore $\card{\Gamma(v)\cap S} \le \card{\Gamma(v)\cap F(u)} + |T| \le d/10$. Hence $v'$ is not in $S'$.

Following these arguments, $S'$ must be equal to $F(u)$, and the algorithm successfully learns the edges related to $u$.

The algorithm will successfully find all vertices $u\in U$ because of Property 4: for every $u$ there are enough number of edges in $G$ that is only caused by $u$. When one of them is sampled, the algorithm successfully learns the vertex $u$.

Finally we bound the running time. By Property 4 we know that the algorithm identifies a new vertex $u\in U$ in at most $10$ iterations in expectation. Each iteration takes at most $O(n)$ time. Therefore the algorithm takes at most $O(n^2)$ time in expectation.
\end{proof}

%

\begin{algorithm}
\caption{RecoverGraph}\label{alg:community_finding_simple}
\begin{algorithmic}
\REQUIRE $G$ given as in Definition~\ref{def:graphrecovery}
\ENSURE Find the graph $G_1$ as in Definition~\ref{def:graphrecovery}.
\REPEAT
\STATE Pick a random edge $(v_1,v_2) \in E$.
\STATE Let $S = \{v : (v,v_1), (v,v_2) \in E \}$. 
\IF {$\card{S} < 1.3d$}
\STATE $S' = \{v\in S : \card{\Gamma(v) \cap S} \ge 0.8d-1\}$
\COMMENT{$S'$ should be a clique in $G$}
\STATE In $G_1$, create a vertex $u$ and connect $u$ to every $v\in S'$.
\STATE Mark all the edges $(v_1,v_2)$ for $v_1,v_2\in S'$.
\ENDIF
\UNTIL{all edges are marked}
\end{algorithmic}
\end{algorithm}

\section{Learning the lowermost (real-valued) layer}
\label{sec:lastlayer}


Note that in our model, the lowest (observed) layer is real-valued and does not have threshold gates.
Thus our earlier learning algorithm cannot be applied as is. However, we see that the same paradigm -- identifying correlations and using {\sc Graph recover} -- can be used.


The first step is to show that for a random weighted graph $G$, the linear decoder $D(h) = Gh$ and the encoder $E(y) = \sgn(G^Ty + b)$ form a denoising autoencoder with real-valued outputs, as in Bengio et al.~\cite{DBLP:journals/pami/BengioCV13}. 

\begin{lemma}
If $G$ is a random weighted graph, the encoder $E(y) = \sgn(G^T y-0.4d\vec{1})$ and linear decoder $D(h) = Gh$ form a denoising autoencoder, for noise vectors $\gamma$ which have independent components, each having variance at most $O(d/\log^2 n)$.
%
\end{lemma}

The next step is to show a bound on correlations as before. For simplicity we state it assuming the 
layer $h^{(1)}$ has a random $0/1$ assignment of sparsity $\rho_1$. In the full version we state it
keeping in mind the higher layers, as we did in the previous sections.

\begin{theorem}\label{thm:last_layer_correlation_3wise_simple}
When $\rho_1 d = O(1)$, $d = \Omega(\log^2 n)$, 
with high probability over the choice of the weights and the choice of the graph, for any three nodes $u,v,s$ the assignment $y$ to the bottom layer satisfies:
\begin{enumerate}
	\item If $u,v$ and $s$ have no common neighbor, then $|\Exp_{h}[y_uy_vy_s]|\le \rho_1/3$
	\item If $u,v$ and $s$ have a unique common neighbor, then $|\Exp_{h}[y_uy_vy_s]| \ge 2\rho_1/3$
\end{enumerate}
\end{theorem}

\section{Two layers cannot be represented by one layer}
\label{sec:lowerbounds}

In this section we show that a two-layer network with $\pm 1$ weights is more expressive than one layer network with arbitrary weights. A two-layer network $(G_1,G_2)$ consists of random graphs $G_1$ and $G_2$ with random $\pm 1$ weights on the edges. Viewed as a generative model,
its input is $h^{(3)}$ and the output is $h^{(1)}=\sgn(G_1\sgn(G_2h^{(3)}))$. We will show that
a single-layer network even with arbitrary weights and arbitrary threshold functions must generate a fairly different distribution.

\begin{lemma}\label{lem:2layervs1}
For almost all choices of $(G_1,G_2)$, the following is true. For every one layer network  with matrix $A$ and vector $b$, if $h^{(3)}$  is chosen to be a random $\rho_3 n$-sparse vector with $\rho_3d_2d_1\ll 1$, the probability (over the choice of
$h^{(3)}$) is at least $\Omega(\rho_3^2)$ that $\sgn(G_1\sgn(G_1h^{(3)}))\neq \sgn(Ah^{(3)}+ b)$. 
\end{lemma}

The idea is that the cancellations possible in the two-layer network simply cannot all be accomodated in a single-layer
network even using arbitrary weights. More precisely, even the bit at a single output node $v$ cannot be well-represented by a simple threshold function. 


First, observe that the output at $v$ is determined by values of $d_1d_2$ nodes at the top layer
that are its ancestors. It is not hard to show in the one layer net $(A,b)$, there should be no edge between $v$ and any node $u$ that is not its ancestor. Then consider structure in Figure~\ref{fig:left}. Assuming all other parents of $v$ are 0 (which happen with probability at least $0.9$), and focus on the values of $(u_1,u_2,u_3,u_4)$. When these values are $(1,1,0,0)$ and $(0,0,1,1)$, $v$ is off. When these values are $(1,0,0,1)$ and $(0,1,1,0)$, $v$ is on. This is impossible for a one layer network because the first two ask for $\sum_{A_{u_i,v}} + 2b_v \le 0$ and the second two ask for $\sum_{A_{u_i,v}} + 2b_v < 0$.

\begin{figure}[!ht]
\centering
	\includegraphics*[height=1in]{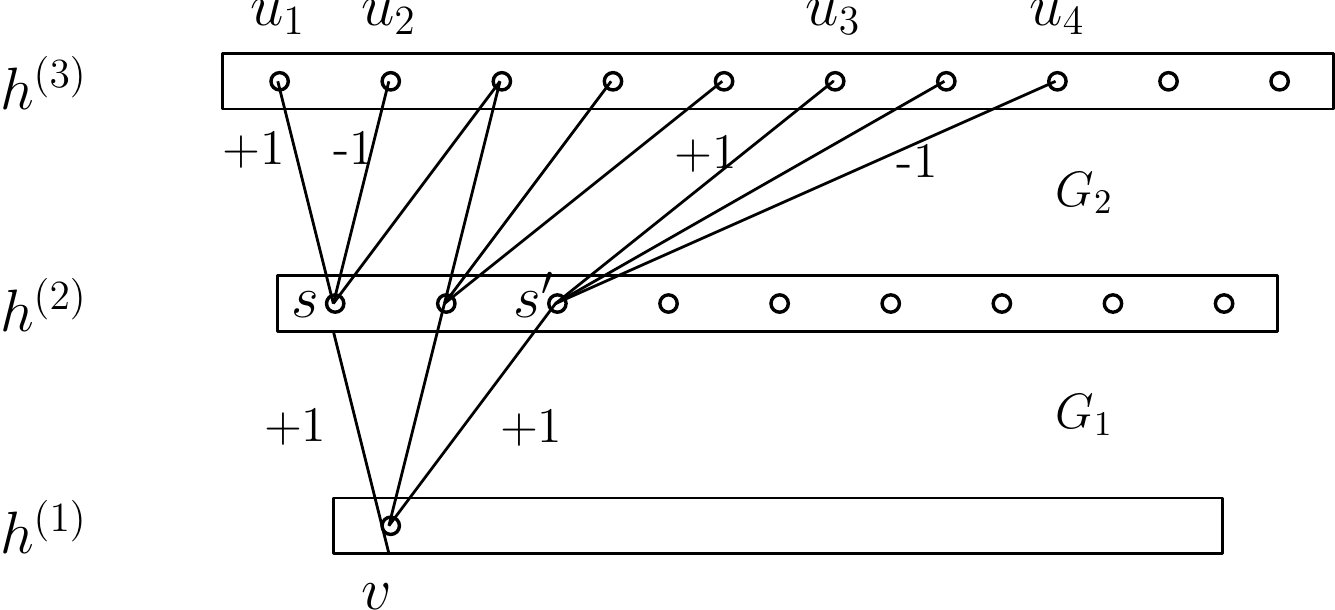}
    \caption{Two-layer network$(G_1,G_2)$}\label{fig:left}
  \end{figure}

\section{Conclusions}
Rigorous analysis of interesting subcases of any ML problem can be beneficial for triggering further improvements: see e.g.,
the role played in Bayes nets
by the rigorous analysis of message-passing algorithms for trees and graphs of low tree-width. 
This is the spirit in which to view our consideration of a random neural net model (though  note that
there is some empirical work in reservoir computing using randomly wired neural nets).

The concept of a denoising autoencoder (with weight tying) suggests to us 
a graph with random-like properties. We would be very interested in an empirical study of the randomness properties
of actual deep nets learnt in real life. (For example, in~\cite{ImageNet} some of the layers use convolution, which is 
decidedly nonrandom. But other layers do backpropagation starting with a complete graph and may end up more random-like.) 

Network randomness is not so crucial for single-layer learning. But for provable layerwise learning we  rely on the 
 support (i.e., nonzero edges) being random: this is crucial for controlling (i.e., upper bounding)
correlations among features appearing in the 
same hidden layer (see Lemma~\ref{lem:corrmanylayer}).
Provable layerwise learning under weaker assumptions would be very interesting.


\section*{Acknowledgments}

We would like to thank Yann LeCun, Ankur Moitra, Sushant Sachdeva, Linpeng Tang for numerous helpful discussions throughout various stages of this work. This work was done when the first, third and fourth authors were visiting EPFL.

\bibliographystyle{alpha}
\bibliography{ref}



\section*{Supplementary material}

# Provable Bounds for Learning Some Deep Representations: Long Technical Appendix[*]

Sanjeev Arora        Aditya Bhaskara        Rong Ge        Tengyu Ma

## A   Preliminaries and Notations

Here we describe the class of randomly chosen neural nets that are learned by our algorithm. A network $\mathcal{R}(\ell, \rho_l, \{G_i\})$ has $\ell$ hidden layers of binary variables $h^{(\ell)}$, $h^{(\ell-1)}$, .., $h^{(1)}$ from top to bottom and an observed layer $x$ at bottom. The set of nodes at layer $h^{(i)}$ is denoted by $N_i$, and $|N_i| = n_i$. For simplicity of analysis, let $n = \max_i n_i$, and assume each $n_i > n^c$ for some positive constant $c$.

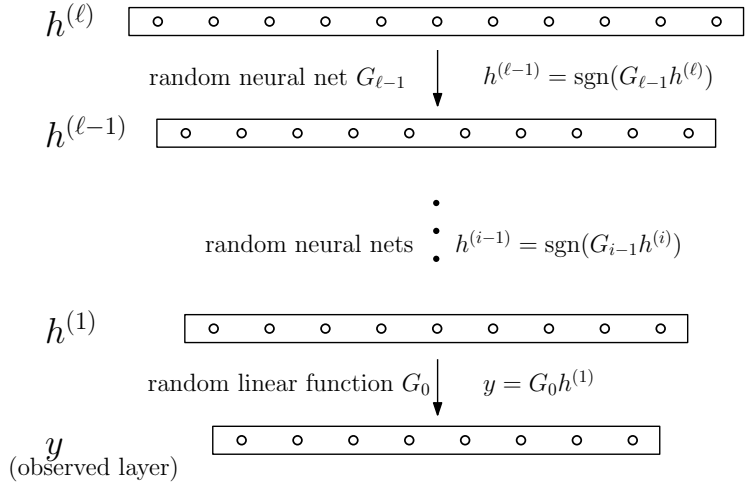

Figure 1: Example of a deep network

The edges between layers $i$ and $i-1$ are assumed to be chosen according to a random bipartite graph $G_i(N_{i+1}, N_i, E_i, w)$ that includes every pair $(u, v) \in N_{i+1} \times N_i$ in $E_i$ with probability $p_i$. We denote this distribution by $\mathcal{G}_{n_{i+1}, n_i, p_i}$. Each edge $e \in E_i$

---

carries a weight $w(e)$ in $[-1, 1]$ that is randomly chosen. The set of positive edges are denoted by $E_i^+ = \{(u, v) \in N_{i+1} \times N_i : w(u, v) > 0\}$. Define $E^-$ to be the negative edges similarly. Denote by $G^+$ and $G^-$ the corresponding graphs defined by $E^+$ and $E^-$, respectively.

The generative model works like a neural net where the threshold at every node is 0. The top layer $h^{(\ell)}$ is initialized a 0/1 assignment where the set of nodes that are 1 is picked uniformly among all sets of size $\rho_l n_l$. Each node in layer $\ell - 1$ computes a weighted sum of its neighbors in layer $\ell$, and becomes 1 iff that sum strictly exceeds 0. We will use $\text{sgn}(\cdot)$ to denote the threshold function:

$$\text{sgn}(x) = 1 \text{ if } x > 0 \text{ and } 0 \text{ else.} \tag{1}$$

Applying $\text{sgn}()$ to a vector involves applying it componentwise. Thus the network computes as follows: $h^{(i-1)} = \text{sgn}(G_{i-1} h^{(i)})$ for all $i > 0$ and $h^{(0)} = G_0 h^{(1)}$ (i.e., no threshold at the observed layer)[1]. Here (with slight abuse of notation) $G_i$ stands for both the bipartite graph and the bipartite weight matrix of the graph at layer $i$.

We also consider a simpler case when the edge weights are in $\{\pm 1\}$ instead of $[-1, 1]$. We call such a network $\mathcal{D}(\ell, \rho_l, \{G_i\})$.

Throughout this paper, by saying "with high probability" we mean the probability is at least $1 - n^{-C}$ for some large constant $C$. Moreover, $f \gg g$ means $f \geq Cg$, $f \ll g$ means $f \leq g/C$ for large enough constant $C$ (the constant required is determined implicitly by the related proofs).

**More network notations.** The expected degree from $N_i$ to $N_{i+1}$ is $d_i$, that is, $d_i \triangleq p_i |N_{i+1}| = p_i n_{i+1}$, and the expected degree from $N_{i+1}$ to $N_i$ is denoted by $d_i' \triangleq p_i |N_i| = p_i n_i$. The set of forward neighbors of $u \in N_{i+1}$ in graph $G_i$ is denoted by $F_i(u) = \{v \in N_i : (u, v) \in E_i\}$, and the set of backward neighbors of $v \in N_i$ in $G_i$ is denoted by $B_i(v) = \{u \in N_{i+1} : (u, v) \in E_i\}$. We use $F_i^+(u)$ to denote the positive neighbors: $F_i^+(u) \triangleq \{v, : (u, v) \in E_i^+\}$ (and similarly for $B_i^+(v)$). The expected density of the layers are defined as $\rho_{i-1} = \rho_i d_{i-1}/2$ ($\rho_\ell$ is given as a parameter of the model).

Our analysis works while allowing network layers of different sizes and different degrees. For simplicity, we recommend first-time readers to assume all the $n_i$'s are equal, and $d_i = d_i'$ for all layers.

**Basic facts about random graphs** We will assume that in our random graphs the expected degree $d', d \gg \log n$ so that most events of interest to us that happen in expectation actually happen with high probability (see Appendix J): e.g., all hidden nodes have backdegree $d \pm \sqrt{d \log n}$. Of particular interest will be the fact (used often

---

[1] We can also allow the observed layer to also use threshold but then our proof requires the output vector to be somewhat sparse. This could be meaningful in modeling practical settings where each datapoint has been represented as a somewhat sparse 0/1 vector via a sparse coding algorithm.

in theoretical computer science) that random bipartite graphs have a *unique neighbor property*. This means that every set of nodes $S$ on one layer has $|S|(d' \pm o(d'))$ neighbors on the neighboring layer provided $|S|d' \ll n$, which implies in particular that most of these neighboring nodes are adjacent to exactly one node in $S$: these are called *unique neighbors*. We will need a stronger version of unique neighbors property which doesn't hold for *all sets* but holds for every set with probability at least $1 - \exp(-d')$ (over the choice of the graph). It says that every node that is not in $S$ shares at most (say) $0.1d'$ neighbors with any node in $S$. This is crucial for showing that each layer is a denoising autoencoder.

# B Main Results

In this paper, we give an algorithm that learns a random deep neural network.

THEOREM 1
For a network $\mathcal{D}(\ell, \rho_l, \{G_i\})$, if all graphs $G_i$'s are chosen according to $\mathcal{G}_{n_{i+1}, n_i, p_i}$, and the parameters satisfy:

1. All $d_i \gg \log^2 n, d'_i \gg \log^2 n$.

2. For all but last layer $(i \geq 1)$, $\rho_i^3 \ll \rho_{i+1}$.

3. For all layers, $n_i^3 (d'_{i-1})^8 / n_{i-1}^8 \ll 1$.

4. For last layer, $\rho_1 d_0 = O(1)$, $d_0^{3/2} / d_1 d_2 < O(\log^{-3/2} n)$, $\sqrt{d_0}/d_1 < O(\log^{-3/2} n)$, $d_1^5 < n$, $d_0 \gg \log^3 n$.

Then there is an algorithm using $O(\log n / \rho_\ell^2)$ samples, running in time $O(\sum_{i=1}^\ell n_i((d'_i)^3 + n_{i-1}))$ that learns the network with high probability on both the graph and the samples.

REMARK 1 We include the last layer whose output is real instead of $0/1$, in order to get fully dense outputs. We can also learn a network without this layer, in which case the last layer needs to have density at most $1/\text{poly}\log(n)$, and condition 4 is no long needed.

REMARK 2 If a stronger version of condition 2, $\rho_i^2 \ll \rho_{i+1}$ holds, there is a faster and simpler algorithm that runs in time $O(n^2)$.

Although we assume each layer of the network is a random graph, we are not using all the properties of the random graph. The properties of random graphs we need are listed in Section J.

We can also learn a network even if the weights are not discrete.

**Theorem 2**
*For a network $\mathcal{R}(\ell, \rho_l, \{G_i\})$, if all graphs $G_i$'s are chosen according to $\mathcal{G}_{n_{i+1}, n_i, p_i}$, and the parameters satisfy the same conditions as in Theorem 1, there is an algorithm using $O(n_l^2 n_{l-1} l^2 \log n / \eta^2)$ samples, running in time $poly(n)$ that learns a network $\mathcal{R}'(\ell, \rho_l, \{G_i'\})$. The observed vectors of network $\mathcal{R}'$ agrees with $\mathcal{R}(\ell, \rho_l, \{G_i\})$ on $(1 - \eta)$ fraction of the hidden variable $h^{(l)}$.*

# C    Each layer is a Denoising Auto-encoder

Experts feel that deep networks satisfy some intuitive properties. First, intermediate layers in a deep representation should approximately preserve the useful information in the input layer. Next, it should be possible to go back/forth easily between the representations in two successive layers, and in fact they should be able to use the neural net itself to do so. Finally, this process of translating between layers should be noise-stable to small amounts of random noise. All this was implicit in the early work on RBM and made explicit in the paper of Vincent et al. [VLBM08] on *denoising autoencoders*. For a theoretical justification of the notion of a denoising autoencoder based upon the "manifold assumption" of machine learning see the survey of Bengio [Ben09].

**Definition 1 (Denoising autoencoder)** *An* autoencoder *consists of an* decoding function $D(h) = s(Wh + b)$ *and a* encoding function $E(y) = s(W'y + b')$ *where $W, W'$ are linear transformations, $b, b'$ are fixed vectors and $s$ is a nonlinear function that acts identically on each coordinate. The autoencoder is* denoising *if $E(D(h)+\eta) = h$ with high probability where $h$ is drawn from the input distribution, $\eta$ is a noise vector drawn from the noise distribution, and $D(h) + \eta$ is a shorthand for "$E(h)$ corrupted with noise $\eta$." The autoencoder is said to use* weight tying *if $W' = W^T$.*

The popular choices of $s$ includes logistic function, soft max, etc. In this work we choose $s$ to be a simple threshold on each coordinate (i.e., the test $> 0$, this can be viewed as an extreme case of logistic function). Weight tying is a popular constraint and is implicit in RBMs. Our work also satisfies weight tying.

In empirical work the denoising autoencoder property is only implicitly imposed on the deep net by minimizing the reconstruction error $||y - D(E(\tilde{y}))||$, where $\tilde{y}$ is a corrupted version of $y$; our definition is very similar in spirit that it also enforces the noise-stability of the autoencoder in a stronger sense. It actually implies that the reconstruction error corresponds to the noise from $\tilde{y}$, which is indeed small. We show that in our ground truth net (whether from model $\mathcal{D}(\ell, \rho_\ell, \{G_i\})$ or $\mathcal{R}(\ell, \rho_\ell, \{G_i\})$) every pair of successive levels whp satisfies this definition, and with weight-tying.

We will show that each layer of our network is a denoising autoencoder with very high probability. (Each layer can also be viewed as an RBM with an additional energy term to ensure sparsity of $h$.) Later we will of course give efficient algorithms to learn

such networks without recoursing to local search. In this section we just prove they satisfy Definition 1.

The single layer has $m$ hidden and $n$ output (observed) nodes. The connection graph between them is picked randomly by selecting each edge independently with probability $p$ and putting a random weight on it in $[-1, 1]$. Then the linear transformation $W$ corresponds simply to this matrix of weights. In our autoencoder we set $b = \vec{0}$ and $b' = 0.2d' \times \vec{1}$, where $d' = pn$ is the expected degree of the random graph on the hidden side. (By simple Chernoff bounds, every node has degree very close to $d'$.) The hidden layer $h$ has the following prior: it is given a $0/1$ assignment that is 1 on a random subset of hidden nodes of size $\rho m$. This means the number of nodes in the output layer that are 1 is at most $\rho m d' = \rho n d$, where $d = pm$ is the expected degree on the observed side. We will see that since $b = \vec{0}$ the number of nodes that are 1 in the output layer is close to $\rho m d'/2$.

**LEMMA 3**
*If $\rho m d' < 0.05n$ (i.e., the assignment to the observed layer is also fairly sparse) then the single-layer network above is a denoising autoencoder with high probability (over the choice of the random graph and weights), where the noise distribution is allowed to flip every output bit independently with probability $0.01$.*

*Remark:* The parameters accord with the usual intuition that the information content must *decrease* when going from observed layer to hidden layer.

PROOF: By definition, $D(h) = \mathrm{sgn}(Wh)$. Let's understand what $D(h)$ looks like. If $S$ is the subset of nodes in the hidden layer that are 1 in $h$, then the unique neighbor property (Corollary 30) implies that (i) With high probability each node $u$ in $S$ has at least $0.9d'$ neighboring nodes in the observed layer that are neighbors to no other node in $S$. Furthermore, at least $0.44d'$ of these are connected to $u$ by a positive edge and $0.44d'$ are connected by a negative edge. All $0.44d'$ of the former nodes must therefore have value 1 in $D(h)$. Furthermore, it is also true that the total weight of these $0.44d'$ positive edges is at least $0.21d'$. (ii) Each $v$ not in $S$ has at most $0.1d'$ neighbors that are also neighbors of any node in $S$.

Now let's understand the encoder, specifically, $E(D(h))$. It assigns 1 to a node in the hidden layer iff the weighted sum of all nodes adjacent to it is at least $0.2d'$. By (i), every node in $S$ must be set to 1 in $E(D(h))$ and no node in $\overline{S}$ is set to 1. Thus $E(D(h)) = h$ for most $h$'s and we have shown that the autoencoder works correctly. Furthermore, there is enough margin that the decoding stays stable when we flip $0.01$ fraction of bits in the observed layer. □

# D    Learning a single layer network

We first consider the question of learning a single layer network, which as noted amounts to learning *nonlinear* dictionaries. It perfectly illustrates how we leverage the sparsity and the randomness of the *support graph*.

The overall algorithm is illustrated in Algorithm 1.

---

**Algorithm 1.** High Level Algorithm

---

**Input:**  samples $y$'s generated by a deep network described in Section A

**Output:**  Output the network/encoder and decoder functions

1: **for** $i = 1$ TO $l$ **do**
2:     Call LastLayerGraph/PairwiseGraph/3-Wise Graph on $h^{(i-1)}$ to construct the correlation structure
3:     Call RecoverGraphLast/RecoverGraph/RecoverGraph3Wise to learn the positive edges $E_i^+$
4:     Use PartialEncoder to encode all $h^{(i-1)}$ to $h^{(i)}$
5:     Call LearnGraph/LearnDecoder to learn the graph/decoder between layer $i$ and $i - 1$.
6: **end for**

---

In Section D.1.1 we start with the simplest subcase: all edge weights are 1 (nonedges may be seen as 0-weight edges). First we show how to use pairwise or 3-wise correlations of the observed variables to figure out which pairs/triples "wire together"(i.e., share a common neighbor in the hidden layer). Then the correlation structure is used by the Graph Recovery procedure (described later in Section F) to learn the support of the graph.

In Section D.1.2 we show how to generalize these ideas to learn single-layer networks with both positive and negative edge weights.

In Section D.2 we show it is possible to do encoding even when we only know the support of positive edges. The result there is general and works in the multi-layer setting.

Finally we give a simple algorithm for learning the negative edges when the edge weights are in $\{\pm 1\}$. This algorithm needs to be generalized and modified if we are working with multiple layers or real weights, see Section G for details.

## D.1   Hebbian rule: Correlation implies common cause

### D.1.1   Warm up: 0/1 weights

In this part we work with the simplest setting: a single level network with $m$ hidden nodes, $n$ observed nodes, and a random (but unknown) bipartite graph $G(U, V, E)$ connecting them where each observed node has expected backdegree degree $d$. All edge weights are 1, so learning $G$ is equivalent to finding the edges. Recall that we denote the hidden variables by $h$ (see Figure 2) and the observed variables by $y$, and the neural network implies $y = \text{sgn}(Gh)$.

Also, recall that $h$ is chosen uniformly at random among vectors with $\rho m$ 1's. The vector is sparse enough so that $\rho d \ll 1$.

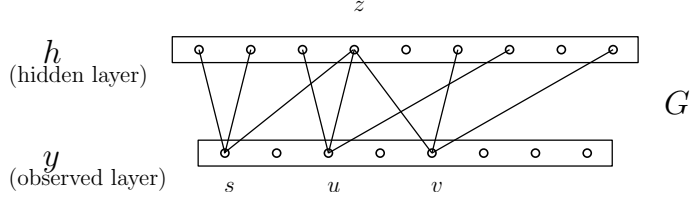

Figure 2: Single layered network

---

**Algorithm 2.** PairwiseGraph

**Input:**   $N = O(\log n/\rho)$ samples of $y = \mathrm{sgn}(Gh)$, where $h$ is unknown and chosen from uniform $\rho m$-sparse distribution

**Output:**   Graph $\hat{G}$ on vertices $V$, $u, v$ are connected if $u, v$ share a positive neighbor in $G$

 **for** each $u, v$ in the output layer **do**

  **if** there are at least $2\rho N/3$ samples of $y$ satisfying both $u$ and $v$ are fired **then**

   connect $u$ and $v$ in $\hat{G}$

  **end if**

 **end for**

---

The learning algorithm requires the unknown graph to satisfy some properties that hold for random graphs with high probability. We summarize these properties as $\mathcal{P}_{sing}$ and $\mathcal{P}_{sing+}$, see Section J.

THEOREM 4
*Let $G$ be a random graph satisfying properties $\mathcal{P}_{sing}$. Suppose $\rho \ll 1/d^2$, with high probability over the samples, Algorithm 2 construct a graph $\hat{G}$, where $u, v$ are connected in $\hat{G}$ iff they have a common neighbor in $G$.*

As mentioned, the crux of the algorithm is to compute the correlations between observed variables. The following lemma shows pairs of variables with a common parent *fire together* (i.e., both get value 1) more frequently than a typical pair. Let $\rho_y = \rho d$ be the approximate expected density of output layer.

LEMMA 5
*Under the assumptions of Theorem 4, if two observed nodes $u, v$ have a common neighbor in the hidden layer then*

$$\Pr_h[y_u = 1, y_v = 1] \geq \rho$$

*otherwise,*

$$\Pr_h[y_u = 1, y_v = 1] \leq 3\rho_y^2$$

PROOF: When $u$ and $v$ has a common neighbor $z$ in the input layer, as long as $z$ is fired both $u$ and $v$ are fired. Thus $\Pr[y_u = 1, y_v = 1] \geq \Pr[h_z = 1] = \rho$.

On the other hand, suppose the neighbor of $u$ ($B(u)$) and the neighbors of $v$ ($B(v)$) are disjoint. Since $y_u = 1$ only if the support of $h$ intersect with the neighbors of $u$, we have $\Pr[y_u = 1] = \Pr[\text{supp}(h) \cap B(u) \neq \emptyset]$. Similarly, we know $\Pr[y_u = 1, y_v = 1] = \Pr[\text{supp}(h) \cap B(u) \neq \emptyset, \text{supp}(h) \cap B(v) \neq \emptyset]$.

Note that under assumptions $\mathcal{P}_{sing}$ $B(u)$ and $B(v)$ have size at most $1.1d$. Lemma 38 implies $\Pr[\text{supp}(h) \cap B(u) \neq \emptyset, \text{supp}(h) \cap B(u) \neq \emptyset] \leq 2\rho^2 |B(u)| \cdot |B(v)| \leq 3\rho_y^2$. $\square$

The lemma implies that when $\rho_y^2 \ll \rho$ (which is equivalent to $\rho \ll 1/(d^2)$), we can find pairs of nodes with common neighbors by estimating the probability that they are both 1.

In order to prove Theorem 4 from Lemma 5, note that we just need to estimate the probability $\Pr[y_u = y_v = 1]$ up to accuracy $\rho/4$, which by Chernoff bounds can be done using by $O(\log n/\rho^2)$ samples.

---

**Algorithm 3.** 3-WiseGraph

**Input:** $N = O(\log n/\rho)$ samples of $y = \text{sgn}(Gh)$, where $h$ is unknown and chosen from uniform $\rho m$-sparse distribution

**Output:** Hypergraph $\hat{G}$ on vertices $V$. $\{u, v, s\}$ is an edge if and only if they share a positive neighbor in $G$

    **for** each $u, v, s$ in the observed layer of $y$ **do**

        **if** there are at least $2\rho N/3$ samples of $y$ satisfying all $u, v$ and $s$ are fired **then**

            add $\{u, v, s\}$ as an hyperedge for $\hat{G}$

        **end if**

    **end for**

---

The assumption that $\rho \ll 1/d^2$ may seem very strong, but it can be weakened using higher order correlations. In the following Lemma we show how 3-wise correlation works when $\rho \ll d^{-3/2}$.

LEMMA 6
*For any $u, v, s$ in the observed layer,*

1. $\Pr_h[y_u = y_v = y_s = 1] \geq \rho$, *if $u, v, s$ have a common neighbor*

2. $\Pr_h[y_u = y_v = y_s = 1] \leq 3\rho_y^3 + 50\rho_y\rho$ *otherwise.*

PROOF: The proof is very similar to the proof of Lemma 5.

If $u, v$ and $s$ have a common neighbor $z$, then with probability $\rho$, $z$ is fired and so are $u$, $v$ and $s$.

On the other hand, if they don't share a common neighbor, then $\Pr_h[u, v, s \text{ are all fired}] = \Pr[\text{supp}(h) \text{ intersects with } B(u), B(v), B(s)]$. Since the graph has property $\mathcal{P}_{sing+}$, $B(u), B(v), B(s)$ satisfy the condition of Lemma 40, and thus we have that $\Pr_h[u, v, s \text{ are all fired}] \leq 3\rho_y^3 + 50\rho_y\rho$. $\square$

### D.1.2 General case: finding common positive neighbors

In this part we show that Algorithm 3 still works even if there are negative edges. The setting is similar to the previous parts, except that the edges now have a random weights. We will only be interested in the sign of the weights, so without loss of generality we assume the nonzero weights are chosen from $\{\pm 1\}$ uniformly at random. All results still hold when the weights are uniformly random in $[-1, 1]$.

A difference in notation here is $\rho_y = \rho d/2$. This is because only half of the edges have positive weights. We expect the observed layer to have "positive" density $\rho_y$ when the hidden layer has density $\rho$.

The idea is similar as before. The correlation $\Pr[y_u = 1, y_v = 1, y_s = 1]$ will be higher for $u, v, s$ with a common *positive cause*; this allows us to identify the $+1$ edges in $G$.

Recall that we say $z$ is a positive neighbor of $u$ if $(z, u)$ is an $+1$ edge, the set of positive neighbors are $F^+(z)$ and $B^+(u)$.

We have a counterpart of Lemma 6 for general weights.

LEMMA 7
*When the graph $G$ satisfies properties $\mathcal{P}_{sing}$ and $\mathcal{P}_{sing+}$ and when $\rho_y \ll 1$, for any $u, v, s$ in the observed layer,*

1. $\Pr_h[y_u = y_v = y_s = 1] \geq \rho/2$, *if $u, v, s$ have a common positive neighbor*

2. $\Pr_h[y_u = y_v = y_s = 1] \leq 3\rho_y^3 + 50\rho_y\rho$, *otherwise.*

PROOF: The proof is similar to the proof of Lemma 6.

First, when $u, v, s$ have a common positive neighbor $z$, let $U$ be the neighbors of $u, v, s$ except $z$, that is, $U = B(u) \cup B(v) \cup B(s) \setminus \{z\}$. By property $\mathcal{P}_{sing}$, we know the size of $U$ is at most $3.3d$, and with at least $1 - 3.3\rho d \geq 0.9$ probability, none of them is fired. When this happens ($\mathrm{supp}(h) \cap U = \emptyset$), the remaining entries in $h$ are still uniformly random $\rho m$ sparse. Hence $\Pr[h_z = 1 \,|\, \mathrm{supp}(h) \cap U = \emptyset] \geq \rho$. Observe that $u, v, s$ must all be fired if $\mathrm{supp}(h) \cap U = \emptyset$ and $h_z = 1$, therefore we know

$$\Pr_h[y_u = y_v = y_s = 1] \geq \Pr[\mathrm{supp}(h) \cap U = \emptyset] \Pr[h_z = 1 \,|\, \mathrm{supp}(h) \cap U = \emptyset] \geq 0.9\rho.$$

On the other hand, if $u$, $v$ and $s$ don't have a positive common neighbor, then we have $\Pr_h[u, v, s \text{ are all fired}] \leq \Pr[\mathrm{supp}(h) \text{ intersects with } B^+(u), B^+(v), B^+(s)]$. Again by Lemma 40 and Property $\mathcal{P}_{mul+}$, we have $\Pr_h[y_u = y_v = y_s = 1] \leq 3\rho_y^3 + 50\rho_y\rho$ □

## D.2 Parital Encoder: Finding $h$ given $y$

Suppose we have a graph generated as described earlier, and that we have found all the positive edges (denoted $E^+$). Then, given $y = \mathrm{sgn}(Gh)$, we show how to recover $h$

as long as it possesses a "strong" unique neighbor property (definition to come). The recovery procedure is very similar to the encoding function $E(\cdot)$ of the autoencoder (see Section C) with graph $E^+$.

Consider a bipartite graph $G(U, V, E)$. An $S \subseteq U$ is said to have the $(1-\epsilon)$-strong unique neighbor property if for each $u \in S$, $(1-\epsilon)$ fraction of its neighbors are unique neighbors with respect to $S$. Further, if $u \notin S$, we require that $|F^+(u) \cap F^+(S)| < d'/4$. Not all sets of size $o(n/d')$ in a random bipartite graph have this property. However, *most* sets of size $o(n/d')$ have this property. Indeed, if we sample polynomially many $S$, we will not, with high probability, see any sets which do not satisfy this property. See Property 1 in Appendix J for more on this.

How does this property help? If $u \in S$, since most of $F(u)$ are unique neighbors, so are most of $F^+(u)$, thus they will all be fired. Further, if $u \notin S$, less than $d'/4$ of the positive neighbors will be fired w.h.p. Thus if $d'/3$ of the positive neighbors of $u$ are on, we can be sure (with failure probability $\exp^{-\Omega(d')}$ in case we chose a bad $S$), that $u \in S$. Formally, the algorithm is simply (with $\theta = 0.3d'$):

---

**Algorithm 4.** PartialEncoder

**Input:**   positive edges $E^+$, sample $y = \mathrm{sgn}(Gh)$, threshold $\theta$
**Output:**   the hidden variable $h$
   **return**   $h = \mathrm{sgn}((E^+)^T y - \theta \vec{1})$

---

**Lemma 8**
*If the support of vector $h$ has the $11/12$-strong unique neighbor property in $G$, then Algorithm 4 returns $h$ given input $E^+$ and $y = \mathrm{sgn}(Gh)$.*

Proof: As we saw above, if $u \in S$, at most $d'/6$ of its neighbors (in particular that many of its positive neighbors) can be shared with other vertices in $S$. Thus $u$ has at least $(0.3)d'$ unique positive neighbors (since $u$ has $d(1 \pm d^{-1/2})$ positive neighbors), and these are all "on".

Now if $u \notin S$, it can have an intersection at most $d'/4$ with $F(S)$ (by the definition of strong unique neighbors), thus there cannot be $(0.3)d'$ of its neighbors that are 1. □

REMARK 3 Notice that Lemma 8 only depends on the unique neighbor property, which holds for the support of any vector $h$ with high probability over the randomness of the graph. Therefore this ParitialEncoder can be used even when we are learning the layers of deep network (and $h$ is not a uniformly random sparse vector). Also the proof only depends on the sign of the edges, so the same encoder works when the weights are random in $\{\pm 1\}$ or $[-1, 1]$.

## D.3   Learning the Graph: Finding −1 edges.

Now that we can find $h$ given $y$, the idea is to use many such pairs $(h, y)$ and the partial graph $E^+$ to determine all the non-edges (i.e., edges of 0 weight) of the graph. Since we know all the $+1$ edges, we can thus find all the $-1$ edges.

Consider some sample $(h, y)$, and suppose $y_v = 1$, for some output $v$. Now suppose we knew that precisely one element of $B^+(v)$ is 1 in $h$ (recall: $B^+$ denotes the back edges with weight $+1$). Note that this is a condition we can *verify*, since we know both $h$ and $E^+$. In this case, it must be that there is no edge between $v$ and $S \setminus B^+$, since if there had been an edge, it must be with weight $-1$, in which case it would cancel out the contribution of $+1$ from the $B^+$. Thus we ended up "discovering" that there is no edge between $v$ and several vertices in the hidden layer.

We now claim that observing polynomially many samples $(h, y)$ and using the above argument, we can discover every non-edge in the graph. Thus the complement is precisely the support of the graph, which in turn lets us find all the $-1$ edges.

---

**Algorithm 5.**  Learning Graph

---

**Input:**   positive edges $E^+$, samples of $(h, y)$, where $h$ is from uniform $\rho m$-sparse distribution, and $y = \text{sgn}(Gh)$

**Output:**   $E^-$

 1: $R \leftarrow (U \times V) \setminus E^+$.
 2: **for** each of the samples $(h, y)$, and each $v$ **do**
 3:     Let $S$ be the support of $h$
 4:     **if** $y_v = 1$ and $S \cap B^+(v) = \{u\}$ for some $u$ **then**
 5:         **for** $s \in S$ **do**
 6:             remove $(s, v)$ from $R$.
 7:         **end for**
 8:     **end if**
 9: **end for**
10: **return**   $R$

---

Note that the algorithm $R$ maintains a set of candidate $E^-$, which it initializes to $(U \times V) \setminus E^+$, and then removes all the non-edges it finds (using the argument above). The main lemma is now the following.

Lemma 9

*Suppose we have $N = O(\log n/(\rho^2 d))$ samples $(h, y)$ with uniform $\rho m$-sparse $h$, and $y = \text{sgn}(Gh)$. Then with high probability over choice of the samples, Algorithm 5 outputs the set $E^-$.*

The lemma follows from the following proposition, which says that the probability that a non-edge $(z, u)$ is identified by one sample $(h, y)$ is at least $\rho^2 d/3$. Thus the probability that it is not identified after $O(\log n/(\rho^2 d))$ samples is $< 1/n^C$. All non-edges must be found with high probability by union bound.

PROPOSITION 10

*Let $(z, u)$ be a non-edge, then with probability at least $\rho^2 d/3$ over the choice of samples, all of the followings hold: 1. $h_z = 1$, 2. $|B^+(u) \cap \text{supp}(h)| = 1$, 3. $|B^-(u) \cap \text{supp}(h)| = 0$.*

*If such $(h, y)$ is one of the samples we consider, $(z, u)$ will be removed from $R$ by Algorithm 5.*

PROOF: The latter part of the proposition follows from the description of the algorithm. Hence we only need to bound the probability of the three events.

Event 1 ($h_z = 1$) happens with probability $\rho$ by the distribution on $h$. Conditioning on 1, the distribution of $h$ is still $\rho m - 1$ uniform sparse on $m - 1$ nodes. By Lemma 39, we have that $\Pr[\text{Event 2 and 3} \mid \text{Event 1}] \geq \rho |B^+(u)|/2 \geq \rho d/3$. Thus all three events happen with at least $\rho^2 d/3$ probability. $\square$

# E  Correlations in a Multilayer Network

We show in this section that Algorithm PairwiseGraph/3-WiseGraph also work in the multi-layer setting. Consider graph $G_i$ in this case, the hidden layer $h^{(i+1)}$ is no longer uniformly random $\rho_{i+1}$ sparse unless $i + 1 = \ell$.[2] In particular, the pairwise correlations can be as large as $\rho_{i+2}$, instead of $\rho_{i+1}^2$. The key idea here is that although the maximum correlation between two nodes in $z, t$ in layer $h^{(i+1)}$ can be large, there are only a few pairs with such high correlation. Since the graph $G_i$ is random and independent of the upper layers, we don't expect to see a lot of such pairs in the neighbors of $u, v$ in $h^{(i)}$.

We make this intuition formal in the following Theorem:

THEOREM 11

*For any $1 \leq i \leq \ell - 1$, and if the network satisfies Property $\mathcal{P}_{mul+}$ with parameters $\rho_{i+1}^3 \ll \rho_i$, then given $O(\log n/\rho_{i+1})$ samples, Algorithm 3 3-WiseGraph constructs a hypergraph $\hat{G}$, where $(u, v, s)$ is an edge if and only if they share a positive neighbor in $G_i$.*

LEMMA 12

*Flor any $i \leq \ell - 1$ and any $u, v, s$ in the layer of $h^{(i)}$ , if they have a common positive neighbor(parent) in layer of $h^{(i+1)}$*

$$\Pr[h_u^{(i)} = h_v^{(i)} = h_s^{(i)} = 1] \geq \rho_{i+1}/3,$$

*otherwise*

$$\Pr[h_u^{(i)} = h_v^{(i)} = h_s^{(i)} = 1] \leq 2\rho_i^3 + 0.2\rho_{i+1}$$

---

[2] Recall that $\rho_i = \rho_{i+1} d_i/2$ is the expected density of layer $i$.

PROOF: Consider first the case when $u, v$ and $s$ have a common positive neighbor $z$ in the layer of $h^{(i+1)}$. Similar to the proof of Lemma 7, when $h_z^{(i+1)} = 1$ and none of other neighbors of $u, v$ and $s$ in the layer of $h^{(i+1)}$ is fired, we know $h_u^{(i)} = h_v^{(i)} = h_s^{(i)} = 1$. However, since the distribution of $h^{(i+1)}$ is not uniformly sparse anymore, we cannot simply calculate the probability of this event.

In order to solve this problem, we go all the way back to the top layer. Let $S = supp(h^{(\ell)})$, and let event $E_1$ be the event that $S \cap B_+^{(\ell)}(u) \cap B_+^{(\ell)}(v) \cap B_+^{(\ell)}(s) \neq \emptyset$, and $S \cap (B^{(\ell)}(u) \cup B^{(\ell)}(v) \cup B^{(\ell)}(s)) = S \cap B_+^{(\ell)}(u) \cap B_+^{(\ell)}(v) \cap B_+^{(\ell)}(s)$ (that is, $S$ does not intersect at any other places except $B_+^{(\ell)}(u) \cap B_+^{(\ell)}(v) \cap B_+^{(\ell)}(s)$). By the argument above we know $E_1$ implies $h_u^{(i)} = h_v^{(i)} = h_s^{(i)} = 1$.

Now we try to bound the probability of $E_1$. Intuitively, $B_+^{(\ell)}(u) \cap B_+^{(\ell)}(v) \cap B_+^{(\ell)}(s)$ contains $B_+^{(\ell)}(z)$, which is of size roughly $d_{i+1} \ldots d_{\ell-1}/2^{\ell-i-1} = \rho_{i+1}/\rho_\ell$. On the other hand, $B^{(\ell)}(u) \cup B^{(\ell)}(v) \cup B^{(\ell)}(s)$ is of size roughly $3d_i \ldots d_{\ell-1} \approx 2^{\ell-i}\rho_i/\rho_\ell$. These numbers are still considerably smaller than $1/\rho_\ell$ due to our assumption on the sparsity of layers ($\rho_i \ll 1$). Thus by applying Lemma 39 with $T_1 = B_+^{(\ell)}(z)$ and $T_2 = B^{(\ell)}(u) \cup B^{(\ell)}(v) \cup B^{(\ell)}(s)$, we have

$$\Pr[h_u^{(i)} = h_v^{(i)} = h_s^{(i)} = 1] \geq \Pr[S \cap T_1 \neq \emptyset, S \cap (T_2 - T_1) = \emptyset] \geq \rho_\ell |T_1|/2 \geq \rho_{i+1}/3,$$

the last inequality comes from Property $\mathcal{P}_{mul}$.

On the other hand, if $u, v$ and $s$ don't have a common positive neighbor in layer of $h^{(i+1)}$, consider event $E_2$: $S$ intersects each of $B_+^{(\ell)}(u)$, $B_+^{(\ell)}(v)$, $B_+^{(\ell)}(s)$. Clearly, the target probability can be upperbounded by the probability of $E_2$. By the graph properties we know each of the sets $B_+^{(\ell)}(u), B_+^{(\ell)}(v), B_+^{(\ell)}(s)$ has size at most $A = 1.2d_i \ldots d_{\ell-1}/2^{\ell-i} = 1.2\rho_i/\rho_\ell$. Also, we can bound the size of their intersections by graph property $\mathcal{P}_{mul}$ and $\mathcal{P}_{mul+}$: $\left|B_+^{(\ell)}(u) \cap B_+^{(\ell)}(v)\right| \leq B = 10\rho_{i+1}/\rho_\ell$, $\left|B_+^{(\ell)}(u) \cap B_+^{(\ell)}(v) \cap B_+^{(\ell)}(s)\right| \leq C = 0.1\rho_{i+1}/\rho_\ell$. Applying Lemma 40 with these bounds, we have

$$\Pr[E_2] \leq \rho_\ell^3 A^3 + 3\rho_\ell^2 AB + \rho_\ell C \leq 2\rho_i^3 + 0.2\rho_{i+1},$$

□

# F  Graph Recovery

Graph reconstruction consists of recovering a graph given information about its subgraphs. A prototypical problem is the *Graph Square Root* problem, which calls for recovering a graph given all pairs of nodes whose distance is at most 2. This is NP-hard. Our setting is a subcase of Graph Square root, whereby there is an unknown bipartite graph and we are told for all pairs of nodes on one side whether they have distance 2. This is also NP-hard in general but is solvable when the bipartite graph

is random or "random-like". Recall that we apply this algorithm to all positive edges between the hidden and observed layer.

**Definition 2 (Graph Recovery Problem)** *There is an unknown random bipartite graph $G_1(U, V, E_1)$ between $|U| = m$ and $|V| = n$ vertices. Every edge is chosen with probability $\hat{d}'/n$.*
Given: *Graph $\hat{G}(V, E)$ where $(v_1, v_2) \in E$ iff $v_1$ and $v_2$ share a common parent in $G_1$ (i.e. $\exists u \in U$ where $(u, v_1) \in E_1$ and $(u, v_2) \in E_1$).*
Goal: *Find the bipartite graph $G_1$.*

Since $U$ and $V$ are just the last two layers in the deep network, we adapt the notations in previous sections. We use $F(u)$ to denote the forward neighbors $\{v \in V : (u, v) \in E_1\}$ and $B(v)$ to denote the backward neighbors $\{u \in U : (u, v) \in E_1\}$. For sets of vertices, let $F(S) = \cup_{u \in S} F(u)$ and $B(S) = \cup_{v \in S} B(v)$. We use $\Gamma(v)$ to denote the neighbors of $V$ in $\hat{G}$. In particular, the construction of graph $\hat{G}$ implies $\Gamma(v) = F(B(v))$.

Notice that the graph $\hat{G}$ is the union of cliques on $F(u)$'s. This problem can be thought of as a "community finding" problem: vertices in $U$ are communities and vertices in $V$ are people; two people are connected if they are in the same community.

We shall give an algorithm that works when $m^2 d'^3/n^3 \ll 1$ and show how to improve the guarantee using three-wise correlations.

## F.1    Graph Recovery from Pairwise Correlations

---

**Algorithm 6.**  RecoverGraph

---

**Input:**   $\hat{G}$ given as in Definition 2
**Output:**   Find the graph $G_1$ as in Definition 2.
  **repeat**
    Pick a random edge $(v_1, v_2) \in E$.
    Let $S = \{v : (v, v_1), (v, v_2) \in E\}$.
    **if** $|S| < 1.3d'$ **then**
      $S' = \{v \in S : |\Gamma(v) \cap S| \geq 0.8d' - 1\}$ $\{S'$ should be a clique in $\hat{G}\}$
      In $G_1$, create a vertex $u$ and connect $u$ to every $v \in S'$.
      Mark all the edges $(v_1, v_2)$ for $v_1, v_2 \in S'$.
    **end if**
  **until** all edges are marked

---

The idea of the algorithm is simple: since the graph $G_1$ is randomly generated, and the size of neighborhoods are small, the communities should have very small intersections. In particular, if we pick two random vertices within a community, the intersection of their neighborhoods will almost be equal to this community.

The algorithm uses this idea to find most vertices $S$ of the community that both $v_1$ and $v_2$ belong to, and applies a further refinement to get $S'$. By properties of the graph $S'$ should be exactly equal to the community.

We make these intuitions precise in the following theorem.

**THEOREM 13**
*Graph Recovery When the graph is chosen randomly according to Definition 2, and when $m^2 d'^3/n^3 \ll 1$, with high probability (over the choice of the graph), Algorithm 6 solves Graph Recovery Problem. The expected running time (over the algorithm's internal randomness) is $O(mn)$.*

In order to prove the theorem, we need the following properties from the random graph:

1. For any $v_1, v_2 \in V$, $|(\Gamma(v_1) \cap \Gamma(v_2)) \backslash (F(B(v_1) \cap B(v_2)))| < d'/20$.

2. For any $u_1, u_2 \in U$, $|F(u_1) \cup F(u_2)| > 1.5d'$.

3. For any $u \in U$, $v \in V$ and $v \notin F(u)$, $|\Gamma(v) \cap F(u)| < d'/20$.

4. For any $u \in U$, at least 0.1 fraction of pairs $v_1, v_2 \in F(u)$ does not have a common neighbor other than $u$.

5. For any $u \in U$, its degree is in $[0.8d', 1.2d']$

The first property says "most correlations are generated by common cause": every vertex $v$ in $F(B(v_1) \cap B(v_2))$ is necessarily a neighbor of both $v_1$ and $v_2$ by the definition of graph $\hat{G}$, because they have a common cause $u \in U$. The first property asserts except this case, the number of common neighbors of $v_1$ and $v_2$ is very small.

The second property basically says the sets $F(u)$'s should be almost disjoint, this is clear because the sets are chosen at random.

The third property says if a vertex $v$ is not related to the cause $u$, then it cannot have correlation with all many neighbors of $u$.

The fourth property says every cause introduces a significant number of correlations that is unique to that cause.

In fact, Properties 2-4 all follow from the unique neighbor property in Lemma 29. The last property is very standard because $d' \gg \log n$

**LEMMA 14**
*When the graph is chosen randomly according to Definition 2, and when $m^2 d'^3/n^3 \ll 1$, with probability $e^{-\Omega(d')}$ Properties 1-5 holds.*

**PROOF:** Property 1: Fix any $v_1$ and $v_2$, Consider the graph $G_1$ to be sampled in the following order. First fix the degrees (the degrees are arbitrary between $0.8d'$ and $1.2d'$), then sample the edges related to $v_1$ and $v_2$, finally sample the rest of the graph.

At step 2, let $S_1 = B(v_1) \backslash B(v_2)$, $S_2 = B(v_1) \backslash B(v_2)$. By the construction of the graph $\hat{G}$, every vertex in $(\Gamma(v_1) \cap \Gamma(v_2)) \backslash (F(B(v_1) \cap B(v_2)))$ must be in $F(S_1) \cap F(S_2)$. With high probability ($e^{-\Omega(d')}$) we know $|S_1| \leq |B(v_1)| \leq 2md/n$ (this is by Chernoff bound, because each vertex $u$ is connected to $v$ with probability $d_u/n < 1.2d'/n$). Similar things hold for $S_2$.

Now $F(S_1)$ and $F(S_2)$ are two random sets of size at most $3m(d')^2/n$, thus again by Chernoff bound we know their intersection has size smaller than $10(md'^2/n)^2/n$, which is smaller than $d'/20$ by assumption.

Property 2: This is easy, for any two vertices $u_1$ and $u_2$, $F(u_1)$ and $F(u_2)$ are two random sets of size at most $1.2d'$. Their expected intersection is less than $2(d')^2/n$ and the probability that their intersection is larger than $0.1d'$ is at most $n^{-\Omega(0.1d')}$. Therefore $|F(u_1) \cup F(u_2)| = |F(u_1)| + |F(u_2)| - |F(u_1) \cap F(u_2)| \geq 1.5d'$.

Property 3: Consider we sample the edges related to $u$ in the last step. Before sampling $F(u)$ we know $v \notin F(u)$, and $\Gamma(v)$ is already a fixed set of size at most $3md'^2/n$. The probability that $F(u)$ has more than $d'/20$ elements in $\Gamma(v)$ is at most $e^{-\Omega(d')}$ by Chernoff bounds.

Property 4: Again change the sampling process: first sample all the edges not related to $u$, then sample $1/2$ of the edges connecting to $u$, and finally sample the second half.

Let $S_1$ be the first half of $F(u)$. For a vertex outside $S_1$, similar to property 3 we know every $v \notin S_1$ has at most $d'/20$ neighbors in $S_1$, therefore any new sample in the second half is going to introduce $0.8d'/2 - d'/20$ new correlations that are unique to $u$. The total number of new correlations is at least $(0.8d'/2 - d'/20)d_u/2 > 0.1\binom{d_u}{2}$.

Property 5 follows from simple concentration bounds. $\square$

Now we show that when the graph satisfies all these properties, the algorithm works.

**Lemma 15**
*When graph $G_1$ satisfies Properties 1-5, Algorithm 6 successfully recovers the graph $G_1$ in expected time $O(mn)$.*

**Proof:** We first show that when $(v_1, v_2)$ has more than one unique common cause, then the condition in the if statement must be false. This follows from Property 2. We know the set $S$ contains $F(B(v_1) \cap B(v_2))$. If $|B(v_1) \cap B(v_2)| \geq 2$ then Property 2 says $|S| \geq 1.5d'$, which implies the condition in the if statement is false.

Then we show if $(v_1, v_2)$ has a unique common cause $u$, then $S'$ will be equal to $F(u)$. By Property 1, we know $S = F(u) \cup T$ where $|T| \leq d'/20$.

For any vertex $v$ in $F(u)$, it is connected to every other vertex in $F(u)$. Therefore $|\Gamma(v) \cap S| \geq |\Gamma(v) \cap F(u)| \geq 0.8d' - 1$, and $v$ must be in $S'$.

For any vertex $v'$ outside $F(u)$, by Property 3 it can only be connected to $d'/20$ vertices in $F(u)$. Therefore $|\Gamma(v) \cap S| \leq |\Gamma(v) \cap F(u)| + |T| \leq d'/10$. Hence $v'$ is not in $S'$.

Following these arguments, $S'$ must be equal to $F(u)$, and the algorithm successfully learns the edges related to $u$.

The algorithm will successfully find all vertices $u \in U$ because of Property 4: for every $u$ there are enough number of edges in $\hat{G}$ that is only caused by $u$. When one of them is sampled, the algorithm successfully learns the vertex $u$.

Finally we bound the running time. By Property 4 we know that the algorithm identifies a new vertex $u \in U$ in at most 10 iterations in expectation. Each iteration takes at most $O(n)$ time. Therefore the algorithm takes at most $O(mn)$ time in expectation. $\square$

## F.2 Graph Recovery with Higher-Order Correlations

When $d$ becomes larger, it becomes harder to recover the graph from only pairwise correlations because more and more pairs of nodes on the other side are at distance 2. In particular if $d'^2 > n$ then almost all pairs are at distance 2 and the graph $\hat{G}$ given to us is simply a complete graph and thus reveals no information about $G_1$.

We show how to use higher-order correlations in order to improve the dependency on $d$. For simplicity we only use 3-wise correlation, but this can be extended to higher order correlations.

DEFINITION 3 (GRAPH RECOVERY WITH 3-WISE CORRELATION) *There is an unknown random bipartite graph $G_1(U, V, E_1)$ between $|U| = m$ and $|V| = n$ vertices. Every edge is chosen randomly with probability $d'/n$.*
Given: *Hypergraph $\hat{G}(V, E)$ where $(v_1, v_2, v_3) \in E$ iff there exists $u \in U$ where $(u, v_1), (u, v_2)$ and $(u, v_3)$ are all in $E_1$.*
Goal: *Find the bipartite graph $G_1$.*

---

**Algorithm 7.** RecoverGraph3Wise

---

**Input:**  Hypergraph $\hat{G}$ in Definition 3
**Output:**  Graph $G_1$ in Defintion 3
  **repeat**
    Pick a random hyperedge $(v_1, v_2, v_3)$ in $E$
    Let $S = \{v : (v, v_1, v_2), (v, v_1, v_3), (v, v_2, v_3) \in E\}$
    **if** $|S| < 1.3d'$ **then**
      Let $S' = \{v \in S : v$ is correlated with at least $\binom{0.8d'-1}{2}$ pairs in $S\}$
      In $G_1$, create a vertex $u$ and connect $u$ to every $v \in S'$.
      Mark all hyperedges $(v_1, v_2, v_3)$ for $v_1, v_2, v_3 \in S'$
    **end if**
  **until** all hyperedges are marked

---

The intuitions behind Algorithm 7 are very similar to Algorithm 6: since 3-wise correlations are rare, not many vertices should have 3-wise correlation with all three pairs $(v_1, v_2)$, $(v_1, v_3)$ and $(v_2, v_3)$ unless they are all in the same community. The

performance of Algorithm 7 is better than the previous algorithm because 3-wise correlations are rarer.

## Theorem 16

*When the graph is chosen according to Definition 3, and when $m^3 d'^8/n^8 \ll 1$, with high probability over the randomness of the graph, Algorithm 7 solves Graph Recovery with 3-wise Correlation. The expected running time is $O(m(d'^3 + n))$ over the randomness of the algorithm.*

The theorem uses very similar properties as Theorem 13, but the pairwise correlations are replaced by three-wise correlations:

1' For any $(v_1, v_2, v_3) \in E$, if $S$ is the set defined as in the algorithm, then $|S \backslash F(B(v_1) \cap B(v_2) \cap B(v_3))| < d'/20$.

2' For any $u_1, u_2 \in U$, $|F(u_1) \cup F(u_2)| > 1.5d'$.

3' For any $u \in U$, $v \in V$ and $v \notin F(u)$, $v$ is correlated with at most $d'^2/40$ pairs in $F(u)$.

4' For any $u \in U$, at least 0.1 fraction of triples $v_1, v_2, v_3 \in F(u)$ does not have a common neighbor other than $u$.

5' For any $u \in U$, its degree is in $[0.8d', 1.2d']$

We can use concentration bounds to show all these properties hold for a random graph.

## Lemma 17

*When the graph is chosen randomly according to Definition 3, and when $m^3 d'^8/n^8 \ll 1$, with probability $e^{-\Omega(d')}$, graph $G_1$ satisfies property $1' - 5'$.*

Now the following lemma implies the theorem.

## Lemma 18

*When the graph $G_1$ satisfies properties $1' - 5'$, Algorithm 7 finds $G_1$ in expected time $O(m(d'^3 + n))$.*

Proof: The idea is again very similar to the pairwise case (Lemma 15).

If $(v_1, v_2, v_3)$ has more than one common neighbor, then Property 2' shows the condition in if statement must be false (as $|S| \geq |F(B(v_1) \cap B(v_2) \cap B(v_3))| \geq 1.5d'$.

When $(v_1, v_2, v_3)$ has only one common neighbor $u$, then Property 1' shows $S = F(u) \cup T$ where $|T| \leq d'/20$.

Now consider $S'$, for any $v \in F(u)$, it is correlated with all other pairs in $F(u)$. Hence it must be correlated with at least $\binom{0.8d'-1}{2}$ pairs in $S$, which implies $v$ is in $S$.

36

For any $v' \notin F(u)$, by Property 3' it can only be correlated with $d'^2/40$ pairs in $F(u)$. Therefore, the total number of correlated pairs it can have in $S$ is bounded by $|T| |F(u)| + \binom{|T|}{2} + d'^2/40 < \binom{0.8d'-1}{2}$. This implies $v'$ is not in $S$.

The argument above shows $S' = F(u)$, so the algorithm correctly learns the edges related to $u$.

Finally, to bound the running time, notice that Property 4' shows the algorithm finds a new vertex $u$ in 10 iterations in expectation. Each iteration takes at most $n + d'^3$ time. Therefore the algorithm takes $O(m(d'^3 + n))$ expected time. $\square$

# G   Learning the graph

In this section we show how to learn the graph $G_t$ ($t = 1, 2, ..., \ell$). For simplicity, we focus on the layer above the bottommost layer (the last layer with discrete output). Recall the graph $G(U, V, E, w)$ is chosen from the distribution $\mathcal{G}_{m,n,p}$ ($m = n_2$ and $n = n_1$). The vector $h$ on the $U$ side is the $h^{(2)}$ in the original network, and is generated according to the distribution $\mathcal{D}_2$. The vector $y$ on $V$ side is $h^{(1)}$ in the original network, and $y = \text{sgn}(Gh)$. As usual, we denote the expected degree of nodes in $U$ as $d' = pn$, and the expected degree of nodes in $V$ as $d = pm$. Let $\rho$ be the sparsity on the $U$ side, $\rho_y$ be the sparsity on the $V$ side. The $V$ side is also relatively sparse in the sense that $\rho_y d \ll 1$.

At this step, we have already applied PartialEncoder to $y$, and Lemma 8 ensures we get the correct vector $h$. Hence in this Section we assume we are given pairs $(h, y)$.

## G.1   Learning the graph for $\mathcal{D}(\ell, \rho_l, \{G_i\})$

If the edges have weights $\pm 1$, we have shown how to learn the negative edges $E^-$ in Section D. The only difference here is that $h$ is chosen from $\mathcal{D}_2$ instead of random $\rho_2 n_2$ sparse vectors. We prove the following Lemma that replaces Lemma 9.

LEMMA 19
*Suppose $\rho_y d \ll 1$, using $O(\log n/(\rho^2 d))$ samples, with high probability over the choice of the samples, Algorithm 5 returns $E^-$.*

PROOF: Similar as in Lemma 9, it suffices to prove the following Proposition 20. Once this Proposition is true, any $(s, v) \notin E$ will be removed from $R$ with high probability. By union bound all non-edges are removed, and the remaining pairs are the edges $E^-$.

PROPOSITION 20
*Suppose $h$ is from $\mathcal{D}_2$ as defined in beginning of this Section, for any $(x, u) \notin E$, with probability $\Omega(\rho^2 d)$ over the choice of $h$, the following events happen simultaneously: 1. $x \in supp(h)$, 2. $|B^+(u) \cap \text{supp}(h)| = 1$, 3. $|B^-(u) \cap \text{supp}(h)| = 0$.*

*When these events happen, $(x, u)$ is removed from $R$ by Algorithm 5.*

PROOF: The idea is very similar to the case when $h$ is chosen from uniformly $\rho_2 n_2$ sparse vectors. Note that event 1 happens with roughly probability $\rho$ the expected fractions of 1's in $h$ is $\rho$. More precisely, for any $z$ in the second layer, $\Pr[h_z = 1] \geq \rho_\ell |B_+^{(\ell)}(z)| = \rho_2 = \rho$. Similarly, event 2 happens if for some nodes $t \in B^+(u)$, $\left| \text{supp}(h^{(\ell)}) \cap B_+^{(\ell)}(t) \right| \geq 1$, and event 3 happens when $B_+^{(\ell)}(B^-(u)) \cap \text{supp}(h^{(\ell)}) = \emptyset$. Using Property $\mathcal{P}_{mul}$ and Lemma 39, one can show that actually $h$ behaves very similarly to $\rho_2 n_2$-sparse vector, where the probability that both the three events happen is $\Omega(\rho^2 d)$.
□
□

## G.2   Learning a Decoder for $\mathcal{R}(\ell, \rho, \{G_i\})$

When the weights on the edges are continuous (in $[-1, 1]$), learning the decoder becomes harder. In particular, later in Section I we show it is hard to learn the weights exactly. Here we give an algorithm that achieves a slightly weaker guarantee: the decoder learned by this algorithm is correct with probability $1 - \eta$.

The key observation here is that every coordinate of $y$ is a *half-plane* on the $h$ vector, so learning the decoder actually reduces to the famous Support Vector Machines.

The hypothesis class for coordinate $y_v$ is simply all the halfplanes $\text{sgn}((Gh)_v)$. By VC-dimension theory, since the $VC$-dimension of a halfplane is $m+1$, any hypothesis that is consistent with all $N$ samples has generalization error $O(\sqrt{(m+1)\log N/N})$. Once we have enough samples, finding a consistent halfplane is a linear program. The full algorithm is given in Algorithm 8

LEMMA 21
*Given $N = O(\ell^2 m n^2 \log n/\eta^2)$ samples of $(h, y)$, where $h$ is chosen from distribution $\mathcal{D}_2$ and $y = \text{sgn}(Gh)$, with high probability over the choice of samples, Algorithm 8 outputs a matrix $G'$ that satisfies $y^i = \text{sgn}(G'h^i)$ for all samples $(y^i, h^i)$. Furthermore,*

$$\Pr_{h \sim \mathcal{D}_t} [\text{sgn}(G'h) \neq \text{sgn}(Gh)] < \eta/\ell.$$

PROOF: It is clear that the LP in Algorithm 8 is feasible, because $G$ is a feasible solution. On the other hand, every feasible solution $G'$ of the LP are consistent with all the samples.

For each output node $v$, since the family of $m$-dimensional half-planes has VC-dimension $m+1$, we know with high probability any row vector $G'^v$ that is consistent with all the $N$ i.i.d. samples has small generalization error

$$\Pr_{h \sim \mathcal{D}_h} [\text{sgn}(G'^v h) \neq \text{sgn}(G^v h)] \leq \sqrt{\frac{2m\log(eN/d)}{N}} + O(\sqrt{\frac{\log(n)}{2N}}) \leq \eta/n\ell$$

38

Taking a union bound for all coordinates of $y$, we have

$$\Pr_{h \sim \mathcal{D}_t}[\mathrm{sgn}(G'h) \neq \mathrm{sgn}(Gh)] < \eta/\ell.$$

□

---

**Algorithm 8.** LearnDecoder

**Input:** $N = O(nl^2m^2 \log n/\eta^2)$ samples $(h^1, y^1), (h^2, y^2), \ldots, (h^N, y^N)$, where $h$ is from distribution $\mathcal{D}_t$ and $y = \mathrm{sgn}(Gh)$

**Output:** A graph $G'$ such that $\Pr_{h \sim \mathcal{D}_t}[\mathrm{sgn}(G'h) \neq \mathrm{sgn}(Gh)] \leq \epsilon$.

Solve the linear program

$$\forall j, \begin{cases} G'h^j \leq 0 & \text{if } y^j = 0 \\ G'h^j > 0 & \text{if } y^j = 1 \end{cases}$$

Here $\leq$ and $>$ are coordinate-wise.

**return** a feasible solution $G'$, the decoder is $y = D(h) = \mathrm{sgn}(G'h)$.

---

When decoders for all layers are learned by Algorithm 8, the deep network composed by the decoders generates a distribution that is close to the original deep network:

**THEOREM 22**

*Given a deep network $\mathcal{R}(\ell, \rho, \{G_i\})$, suppose decoders $D^2, D^3, ..., D^l$ are learned for layers $G_2, G_3, ..., G_\ell$, then for a random $\rho n_\ell$ sparse vector $h^{(l)}$,*

$$\Pr_h[D^2(D^3(\cdots D^l(h^{(l)}) \cdots)) \neq h^{(1)}] \leq \frac{\ell - 1}{\ell} \cdot \eta.$$

*In particular, let $\mathcal{R}'$ be the network generated by stacking the decoders, the outputs of the two networks are $\frac{\ell - 1}{\ell} \cdot \eta$-close in statistical distance.*

**PROOF:** Let bad event $Bad_i$ be the event $h^{(t)} = D^{t+1}(\cdots D^l(h^{(l)}) \cdots)$ for all $t \geq i$, and $h^{(i-1)} \neq D^i(\cdots D^l(h^{(l)}) \cdots)$. Clearly the events are disjoint, and the event of interest is the union of $Bad_2, Bad_3, ..., Bad_l$. By Lemma 21, each $Bad_i$ happens with probability at most $\eta/\ell$. Union bound gives the result.

□

# H   Layer with Real-valued Output

In previous sections, we considered hidden layers with sparse binary input and output. However, in most of applications of deep learning, the observed vector is dense and real-valued. Bengio et al.[BCV13] suggested a variant auto-encoder with linear decoder, which is particularly useful in this case.

We first show for a random weighted graph $G$, the linear decoder function $D(h) = Gh$ and the encoder function $E(y) = \mathrm{sgn}(G^T y + b)$ form a denoising autoencoder.

**THEOREM 23**
*If $G$ is a random graph with random weights in either $\{+1, -1\}$ or $[-1, 1]$, the encoder $E(y) = \text{sgn}(G^T y - 0.1d'\vec{1})$ and linear decoder $D(h) = Wh$ form an sparse autoencoder: that is, for any vector $h$ of support size at most $k$, $E(D(h)) = h$.*

*Further, the autoencoder satisfy the denoising property $E(D(h) + \eta) = h$ when the noise vector $\eta$ has independent components each with variance at most $o(d'_0/\log^2 n)$.*

PROOF: When there is no noise, we know $E(D(\vec{h})) = \text{sgn}(G^T G \vec{h} - 0.2d'_0 \vec{1})$. With high probability the matrix $G^T G$ has value at least $0.9d'_0$ on diagonals (for $\{+1, -1\}$ weights, for $[-1, 1]$ weights this is at least $0.2d'_0$). For any fixed $\vec{h}$, the support of the $i$-th row of $G^T G$ and the support of $\vec{h}$ have a intersection of size at most $d'_0 \log^2 n$ with high probability. Also, at all these intersections, the entries of $G^T G$ has random signs, and variance bounded by $O(1)$. Hence if $\vec{h}_i = 1$ $(G^T G \vec{h})_i \geq 0.9d'_0 - O(\sqrt{d'_0} \log^2 n)$ (or $0.2d'_0 - O(\sqrt{d'_0} \log^2 n)$ for $[-1, 1]$ weights) ; if $\vec{h}_i = 0$ $(G^T G \vec{h})_i \leq O(\sqrt{d'_0} \log^2 n)$. Setting the threshold at $0.1d'_0$ easily distinguishes between these two cases.

Even if we add noise, since the inner product of $G_i$ and the noise vector is bounded by $o(d_0)$ with high probability, the autoencoder still works.□

We use similar ideas as in Section D to learn the last layer. The algorithm collects the correlation-structure of the observed variables, and use this information to reconstruct $E^+$. We inherent most of the notations used in Section C. Formally, the last layer consists of a graph $G(U, V, E, w)$, chosen randomly from distribution $\mathcal{G}_{m,n,p}$. The weights $w$ are chosen uniform randomly from $\pm 1$. Although our result extends to the distribution on $h$ generated by the deep network, and the weights are in $[-1, 1]$ for simplicity, we restrict to the uniform $\rho m$ sparse distribution and random $\pm 1$ weights.

## H.1 Learning network using correlation

When we observe real-valued output in the last layer, we can learn whether three nodes $u$, $v$ and $s$ have a common cause from their correlation measured by $\mathbb{E}[x_u x_v x_s]$, even if the output vector is fully dense. Note that compared to Theorem 4 in Section D, the main difference here is that we allow $\rho pm$ to be any constant (before $\rho pm \ll 1/d$).

**THEOREM 24**
*When $\rho pm = O(1)$, $pn = \Omega(\log^3 n)$, and the input $h$ is a uniformly random $\rho m$-sparse vector, then with high probability over the choice of the weights and the choice of the graph, for any three nodes $u, v, s$ at the output side (side of $y$)*

1. *If $u, v$ and $s$ have no common neighbor, then $|\mathbb{E}_h[y_u y_v y_s]| \leq \rho/3$*

2. *If $u, v$ and $s$ have a unique common neighbor, then $|\mathbb{E}_h[y_u y_v y_s]| \geq 2\rho/3$*

The rigorous proof of Theorem 24 is similar to Lemma 7. We defer the proof in Appendix L.

Although the hypergraph given by Theorem 24 is not exactly the same as the hypergraph given by Theorem 11, the same Algorithm 7 can find the edges of the graph. Notice that in this case we get the support of all edges (instead of just positive edges). Using the sign of $\mathbb{E}_h[y_u y_v y_s]$, we can easily distinguish positive and negative edges (see Appendix L for more details). This idea can be generalized to $[-1, 1]$ weights, and we prove the following theorem in Appendix L.

**THEOREM 25**
*If $h$ is from the distribution $\mathcal{D}_1$, with parameters satisfying $\rho_1 d = O(1)$, $d \gg \log^3 n$), $d^{3/2}/(d_1 d_2) \ll \log^{-3/2} n$ and $\sqrt{d}/d_1 \ll \log^{-3/2} n$, then with high probability over the choice of the weights and choice of the graph, there is an algorithm such that given $O(\log^2 n/\rho_1^2)$ samples of the output $y$, learns the network exactly in time $O(m(d^3 + n) + \log^2 n/\rho_1^2)$.*

**Real weights from $[-1, 1]$** Finally, the above results can be extended to the case when edge weights are uniformly random from $[-1, 1]$. In this case there are several differences:

1. The correlation is large only if the three vertices share a parent with relatively large edge weights. A slight variant of Algorithm 7 learns the large weights.

2. It is still possible to use PartialEncoder even if we only have the large weights.

3. After applying PartialEncoder, we get $(h, y)$ pairs. By solving the system of linear equations $y^i = G h^i$ we learn $G$.

**THEOREM 26**
*If the weight of the network is chosen independently from $[-1, 1]$, and all the others parameters satisfy the same condition as in Theorem 25 (but with different universal constant), there is an algorithm that learns the graph and the weights using $O(\log^2 n/\rho_1^2)$ samples and $O(\log^2 n/\rho_1^2 + m(d^3 + n))$ time.*

# I  Lower bound

## I.1  Improper learning is necessary when weights are reals

In this section, we showcase an example of real weights in which we cannot hope to recover either the weights (even approximately), or to find some set of weights that agrees with the true one on all the inputs, using less than $\exp(\Omega(d))$ samples. This justifies the necessity of the improper/PAC learning of the real weights.

**LEMMA 27**
*There exist two vectors $w, w' \in \{0, 1\}^d$ such that the two functions $f = \text{sgn}(w^T h)$ and $f' = \text{sgn}(w'^T h)$ for $h \in \{0, 1\}^d$ only differ at the point $h = \vec{1}$. Thus it is necessary to*

have $\exp(\Omega(d))$ samples of $h$ from uniform $\rho$-sparse distribution to recover the true value of function $f = \text{sgn}(w^T h)$ at point $h = \vec{1}$.

PROOF: We construct the following $w$ and $w'$. WLOG, let $d$ be an odd number with $d = 2s + 1$. Let $z$ be the vector with first $s$ coordinates being $1/(s+1)$, and the remaining coordinates being $1/s$. Let $w = z + \epsilon\vec{1}$ and $w' = z - \epsilon\vec{1}$, where $0 \le \epsilon < (s(s+1)(2s+1))^{-1}$. Note that $z^T h$ is away from zero by at least $(s(s+1))^{-1}$ if $h \ne \vec{1}$ and $z^T h = 0$ if $h = \vec{1}$. Therefore, $w^T h$ and $w'^T h$ agree on the sign for all $h$ but $h = \vec{1}$. $\square$

## I.2 Representational power: One layer net can't do two layers net

In this section we show that a two-layer network with $\pm1$ weights is more expressive than one layer network with arbitrary weights. A two-layer network $(G_1, G_2)$ consists of random graphs $G_1$ and $G_2$ with random $\pm1$ weights on the edges. Viewed as a generative model, its input is $h^{(3)}$ and the output is $h^{(1)} = \text{sgn}(G_1 \text{sgn}(G_2 h^{(3)}))$. We will show that a single-layer network even with arbitrary weights and arbitrary threshold functions must generate a fairly different distribution.

LEMMA 28
*For almost all choices of $(G_1, G_2)$, the following is true. For every one layer network with matrix $A$ and vector $b$, if $h^{(3)}$ is chosen to be a random $\rho_3 n$-sparse vector with $\rho_3 d_2 d_1 \ll 1$, the probability (over the choice of $h^{(3)}$) is at least $\Omega(\rho_3^2)$ that $\text{sgn}(G_1 \text{sgn}(G_1 h^{(3)})) \ne \text{sgn}(Ah^{(3)} + b)$.*

The idea is that the cancellations possible in the two-layer network simply cannot all be accomodated in a single-layer network even using arbitrary weights. More precisely, even the bit at a single output node $v$ cannot be well-represented by a simple threshold function.

First, observe that the output at $v$ is determined by values of $d_1 d_2$ nodes at the top layer that are its ancestors. Wlog, in the one layer net $(A, b)$, there should be no edge between $v$ and any node $u$ that is not its ancestor. The reason is that these edges between $v$ and its ancestors in $(A, b)$ can be viewed collectively as a single random variables that is not correlated with values at $v$'s ancestors, and either these edges are "influential" with probability at least $\rho_3^2/4$ in which case it causes a wrong bit at $v$; or else it is not influential and removing it will not change the function computed on $\rho_3^2/4$ fraction of probability mass. Similarly, if there is a path from $u$ to $v$ then there must be a corresponding edge in the one-layer network. The question is what weight it should have, and we show that no weight assignment can avoid producing erroneous answers.

The reason is that with probability at least $\rho_3/2$, among all ancestors of $v$ in the input layer, only $u$ is 1. Thus in order to produce the same output in all these cases,

in the one-layer net the edge between $u$ and $v$ should be positive iff the path from $u$ to $v$ consists of two positive edges. But now we show that with high probability there is a cancellation effect involving a local structure in the two layer net whose effect cannot be duplicated by such a single-layer net (See the Figure 3 and 4).

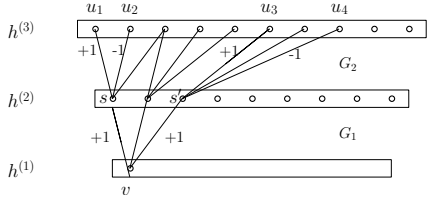

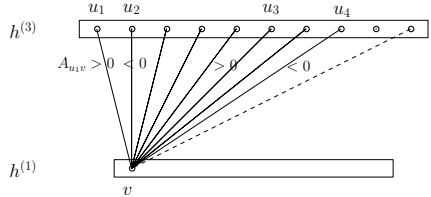

Figure 3: Two-layer network$(G_1, G_2)$      Figure 4: Single-layer network $(A, b)$

As drawn in Figure 3, suppose the nodes $u_1$, $u_2$ connect to $s$ in $h^{(2)}$ via $+1$ and $-1$ edge, and $s$ connects to $v$ via a $+1$ edge. Similarly, the nodes $u_3, u_4$ connect to $s'$ in $h^{(2)}$ via $+1$ and $-1$ edge, and $s$ connects to $v$ via a $+1$ edge.

Now assume all other ancestors of $v$ are off, and consider the following four values of $(u_1, u_2, u_3, u_4)$: $(1, 1, 0, 0)$, $(0, 0, 1, 1)$, $(1, 0, 0, 1)$, $(0, 1, 1, 0)$. In the two-layer network, $h^{(1)}{}_v$ should be 0 for the first two inputs and 1 for the last two inputs. Now we are going to see the contradiction. For single-layer network, these values imply constraints $A_{u_1,v} + A_{u_2,v} + b_v \leq 0$, $A_{u_3,v} + A_{u_4,v} + b_v \leq 0$, $A_{u_1,v} + A_{u_4,v} + b_v > 0$, $A_{u_2,v} + A_{u_3,v} + b_v > 0$. However, there can be no choices of $(A, b)$ that satisfies all four constraints! To see that, simply add the first two and the last two, the left-hand-sides are all $\sum_{i=1}^{4} A_{u_i,v} + 2b_v$, but the right-hand-sides are $\leq 0$ and $> 0$. The single-layer network cannot agree on all four inputs. Each of the four inputs occurs with probability at least $\Omega(\rho_3^2)$. Therefore the outputs of two networks must disagree with probability $\Omega(\rho_3^2)$.

**Remark:** It is well-known in complexity theory that such simple arguments do not suffice to prove lowerbounds on neural nets with more than one layer.

# J    Random Graph Properties

In this Section we state the properties of random graphs that are used by the algorithm. We first describe the unique-neighbor property, which is central to our analysis. In the next part we list the properties required by different steps of the algorithm.

## J.1    Unique neighbor property

Recall that in bipartite graph $G(U, V, E, w)$, the set $F(u)$ denotes the neighbors of $u \in U$, and the set $B(v)$ denotes the neighbors of $v \in V$.

For any node $u \in U$ and any subset $S \subset U$ , let $UF(u, S)$ be the sets of unique neighbors of $u$ with respect to $S$,

$$UF(u, S) \triangleq \{v \in V : v \in F(u), v \notin F(S \setminus \{u\})\}$$

PROPERTY 1 In a bipartite graph $G(U, V, E, w)$, a node $u \in U$ has $(1 - \epsilon)$-unique neighbor property with respect to $S$ if

$$\sum_{v \in UF(u,S)} |w(u, v)| \geq (1 - \epsilon) \sum_{v \in F(u)} |w(u, v)| \tag{2}$$

The set $S$ has $(1 - \epsilon)$-strong unique neighbor property if for every $u \in U$, $u$ has $(1 - \epsilon)$-unique neighbor property with respect to $S$.

REMARK 4 Notice that in our definition, $u$ does not need to be inside $S$. This is not much stronger than the usual definition where $u$ has to be in $S$: if $u$ is not in $S$ we are simply saying $u$ has unique neighbor property with respect to $S \cup \{u\}$. When all (or most) sets of size $|S| + 1$ have the "usual" unique neighbor properties, all (or most) sets of size $|S|$ have the unique neighbor property according to our definition.

LEMMA 29
If $G(U, V, E)$ is from distribution $\mathcal{G}_{m,n,p}$, and for any $u \in U$, the weight vector $w(u, \cdot)$ (i.i.d. on edges) satisfies that $|w(u, \cdot)|_1^2 \geq C \cdot \|w(u, \cdot)\|_2^2$ for some $C$. Then for every subset $S$ of $U$ with $(1 - p)^{|S|} > 1 - \epsilon/2$ (note that $p|S| = d'|S|/n$ is the expected density of $F(S)$), with probability $1 - 2m \exp(-\epsilon^2 C)$ over the randomness of the graph, $S$ has the $(1 - \epsilon)$-strong unique neighbor property.

**Note:** The weights usually comes from one of the following: 1. all weights are 1; 2. weights are random in $\{\pm 1\}$; 3. weights are random in $[-1, 1]$. In all these cases $C$ is $\Theta(d')$.

PROOF:
Fix the vertex $u$, first sample the edges incident to $u$. Without loss of generality assume $u$ has neighbors $v_1, \ldots, v_k$ (where $k \approx pn$). Now sample the edges incident to the $v_i$'s. For each $v_i$, with probability $(1 - p)^{|S|} \geq 1 - \epsilon/2$, $v_i$ is a unique neighbor of $u$. Call this event $Good_i$ (and we also use $Good_i$ as the indicator for this event).

By the construction of the graph we know $Good_i$'s are independent, hence by Hoeffding inequality (see Theorem 45), we have that with probability

$$1 - 2\exp(-\frac{\epsilon^2(\sum_{v \in F(u)} w(u, v))^2}{\sum_{v \in F(u)} w(u, v)^2}) = 1 - 2\exp(-\epsilon^2 |w(u, \cdot)|_1^2/\|w(u, \cdot)\|_2^2) = 1 - 2\exp(-\epsilon^2 C),$$

the following holds

$$\sum_{v \in F(u)} w(u, v) Good_i \geq (1 - \epsilon) \sum_{v \in F(u)} w(u, v).$$

44

By union bound, every $u$ satisfies this property with probability at least $1 - 2m\exp(-\epsilon^2 C)$. $\square$

**COROLLARY 30**
*If $G(U, V, E)$ is chosen from $\mathcal{G}_{m,n,p}$ and all the weights $w(u, v)$ have the same magnitude for $(u, v) \in E$, then for fixed $S$ with $p|S| < \epsilon/2$, with probability $1 - \exp(-\Omega(d'))$, $S$ has the $(1 - \epsilon)$-strong unique neighbor property.*

## J.2 Properties required by each steps

We now list the properties required in our analysis. These properties hold with high probability for random graphs.

### J.2.1 Single layer

The algorithm PairwisGraph requires the following properties $\mathcal{P}_{sing}$

1. For any $u$ in the observed layer, $|B(u)| \in [0.9d, 1.1d]$, (if $G$ has negative weights, we also need $|B^+(u)| \in [0.9d/2, 1.1d/2]$)

2. For any $z$ in the hidden layer, $|F(z)| \in [0.9d', 1.1d']$, (if $G$ has negative weights, we also need $|F^+(z)| \in [0.45d', 0.55d']$)

The algorithm 3-WiseGraph needs $\mathcal{P}_{sing}$, and additionally $\mathcal{P}_{sing+}$

1. For any $u, v$ in the observed layer, $|B^+(u) \cup B^+(v)| \leq 10$.

**LEMMA 31**
*If graph $G$ is chosen from $\mathcal{G}_{m,n,p}$ with expected degrees $d, d' \gg \log n$, with high probability over the choice of the graph, $\mathcal{P}_{sing}$ is satisfied. If in addition $d^2 \leq n^{4/5}$, the property $\mathcal{P}_{sing+}$ is also satisfied with high probability.*

### J.2.2 Multi-layer

For the multi-layer setting, the algorithm PairwisGraph requires the following expansion properties $\mathcal{P}_{mul}$.

1. For any node $u$ at the layer $i$, $|F_{i-1}(u)| \in [0.9d'_{i-1}, 1.1d'_{i-1}]$, $|B_i(u)| \in [0.9d_i, 1.1d_i]$, $\left|F_{i-1}^+(u)\right| \in [0.45d'_{i-1}, 0.55d'_{i-1}]$, $\left|B_i^+(u)\right| \in [0.45d_i, 0.55d_i]$

2. For any node $u$ at the layer $i$, $|B_+^{(t)}(u)| \geq 0.8\rho_i/\rho_t$, and $|B^{(\ell)}(u)| \leq 2^{\ell-i+1}\rho_i/\rho_\ell$.

3. For any pair of nodes $u, v$ at layer $i$,

$$\left|B_+^{(\ell)}(u) \cap B_+^{(\ell)}(v)\right| \leq 2\rho_{i+1}/\rho_\ell \cdot \left(\log n/d_{i+1} + 1/(\rho_\ell n_\ell) + \left|B_+^{(i+1)}(u) \cap B_+^{(i+1)}(v)\right|\right)$$

In particular, if $u$ and $v$ have no common positive parent at layer $i+1$ ($\left|B_+^{(i+1)}(u) \cap B_+^{(i+1)}(v)\right| = 0$), then

$$\left|B_+^{(\ell)}(u) \cap B_+^{(\ell)}(v)\right| \leq o(\rho_{i+1}/\rho_\ell)$$

The algorithm 3-WiseGraph needs property $\mathcal{P}_{mul+}$:

1. Properties 1 and 2 in $\mathcal{P}_{mul}$

2. For any pair of nodes $u, v$ at layer $i$, $\left|B_+^{(\ell)}(u) \cap B_+^{(\ell)}(v)\right| \leq 10\rho_{i+1}/\rho_\ell$

3. For any three nodes $u, v$ and $s$ at layer $i$, if they don't have a common positive neighbor in layer $i+1$,

$$\left|B_+^{(\ell)}(u) \cap B_+^{(\ell)}(v) \cap B_+^{(\ell)}(s)\right|$$
$$\leq 2\rho_{i+1}/\rho_\ell \cdot \left(\log n/d_{i+1} + \rho_\ell/(\rho_\ell n_\ell) + 1/(\rho_\ell^2 n_\ell^2)\right) \leq o(\rho_{i+1}/\rho_\ell)$$

**LEMMA 32**
*If the network $\mathcal{D}(\ell, \rho_\ell, \{G_i\})$ have parameters satisfying $d_i \gg \log n$, and $\rho_i^2 \ll \rho_{i+1}$, then with high probability over the randomness of the graphs, $\{G_i\}$'s satisfy $\mathcal{P}_{mul}$. Additionally, if $d_i \gg \log n$ and $\rho_i^3 \ll \rho_{i+1}$, then $\{G_i\}$'s satisfy $\mathcal{P}_{mul+}$ with high probability.*

In order to prove Lemma 32 we need the following claim:

**CLAIM 33**
*If the graph $G \sim \mathcal{G}_{m,n,p}$ with $d = pm$ being the expected back degree, and $d \gg \log n$. For two arbitrary sets $T_1$ and $T_2$, with $d|T_1| \ll m, d|T_2| \ll m$, we have with high probability*

$$|B(T_1) \cap B(T_2)| \leq (1+\epsilon)d|T_1 \cap T_2| + (1+\epsilon)d^2|T_1||T_2|/m + 5\log n$$

This Claim simply follows from simple concentration bounds. Now we are ready to prove Lemma 32.

PROOF:[Proof of Lemma 32] Property 1 in $\mathcal{P}_{mul}$ follows directly from $d_i \gg \log n$.

Property 2 in $\mathcal{P}_{mul}$ follows from unique neighbor properties (when we view the bipartite graph from $N_i$ to $N_{i+1}$).

For Property 3, we prove the following proposition by induction on $t$:

**PROPOSITION 34**
*For any two nodes $u, v$ at layer 1,*

$$\left|B_+^{(t)}(u) \cap B_+^{(t)}(v)\right| \leq (1+\epsilon)^t t\rho_1^2/(\rho_t^2 n_t) + 6t(1+\epsilon)^t \rho_3 \log n/\rho_t + (1+\epsilon)^t \rho_2 \left|B_+^{(2)}(u) \cap B^{(2)}(v)\right|/\rho_t$$

This is true for $t = 2$ (by Claim 33). Suppose it is true for all the values less than $t - 1$. By Claim 33, we know with high probability (notice that we only need to do union bound on $n^2$ pairs)

$$
\begin{aligned}
\left| B_+^{(t+1)}(u) \cap B_+^{(t+1)}(v) \right| \leq\ & (1+\epsilon)d_t/2 \cdot \rho_2 \left| B_+^{(t)}(u) \cap B_+^{(t)}(v) \right| / \rho_t \\
& + (1+\epsilon)d_t^2/4 \cdot |B_+^{(t)}(u)||B_+^{(t)}(v)| + 5\log n \\
\leq\ & (1+\epsilon)^{t+1}\rho_2 \left| B_+^{(2)}(u) \cap B^{(2)}(v) \right| / \rho_{t+1} + 6t(1+\epsilon)^{t+1}\rho_3 \log n / \rho_{t+1} \\
& + (1+\epsilon)^t t \rho_1^2/(\rho_t^2 n_t) + (1+\epsilon)d_t^2/4 \cdot \rho_1^2/(\rho_t^2 n_{t+1}) + 5\log \\
\leq\ & (1+\epsilon)^{t+1}\rho_2 \left| B_+^{(2)}(u) \cap B^{(2)}(v) \right| / \rho_{t+1} + 6(t+1)(1+\epsilon)^{t+1}\rho_3 \log n / \rho_{t+1} \\
& + 2(1+\epsilon)^{t+1}(t+1)\rho_1^2/(\rho_{t+1}^2 n_{t+1}),
\end{aligned}
$$

where the last inequality uses the fact that $\rho_1^2/(\rho_t^2 n_t) \leq d_t^2/4 \cdot \rho_1^2/(\rho_t^2 n_{t+1})$. This is because $n_t d_t / n_{t+1} = d_t' \gg 1$.

Proposition 34 implies that when $\rho_1^2 \ll \rho_2$, and $\ell$ is a constant,

$$
\left| B_+^{(t)}(u) \cap B_+^{(t)}(v) \right| \leq 2\rho_{i+1}/\rho_\ell \cdot \left( 1/d_{i+1} + 1/(\rho_\ell n_\ell) + \left| B_+^{(i+1)}(u) \cap B_+^{(i+1)}(v) \right| \right)
$$

Property 3 in $\mathcal{P}_{mul+}$ is similar but more complicated. □

### J.2.3   Properties for Graph Reovery

For the algorithm RecoverGraph3Wise to work, the hypergraph generated from the random graph should have the following properties $\mathcal{P}_{recovery+}$.

1' For any $(v_1, v_2, v_3) \in E$, if $S$ is the set defined as in the algorithm, then $|S \backslash F(B(v_1) \cap B(v_2) \cap B(v_3))| < d'/20$.

2' For any $u_1, u_2 \in U$, $|F(u_1) \cup F(u_2)| > 1.5d'$.

3' For any $u \in U$, $v \in V$ and $v \notin F(u)$, $v$ is correlated with at most $d'^2/40$ pairs in $F(u)$.

4' For any $u \in U$, at least 0.1 fraction of triples $v_1, v_2, v_3 \in F(u)$ does not have a common neighbor other than $u$.

5' For any $u \in U$, its degree is in $[0.8d', 1.2d']$

### J.2.4   Properties for Partial Encoder

For the Partial Encoder algorithm to work, we only need the support of $h$ satisfying the strong unique-neighbor property.

### J.2.5 Properties for Learning the graph

In order to learn the $-1$ weights, the conditions we need are similar to $\mathcal{P}_{mul}$ and $\mathcal{P}_{mul+}$.

In the case with real weights, we don't need any assumptions on the graph, because we are relying on VC-dimension arguments.

# K Auxiliary Lemmas for uniform $\rho_\ell n_\ell$ sparse vectors

In this subsection, we provide lemmas about the uniform $\rho n$-sparse vector of dimension $n$.

These Lemmas suggest that when looking at intersections with relatively small sets, uniform $\rho n$ sparse vector behaves like i.i.d. Bernoulli variables with probability $\rho$. Also, in this section the sparse vector is identified by its support $S$.

LEMMA 35
If $S$ is a uniformly random $\rho n$-sparse subset of $[n]$, then for any fix set $T$ with $\rho|T| \ll 1$ and $|T| \ll n$,

$$1 - \rho|T|/2 \geq \Pr[T \cap S = \emptyset] \geq 1 - \rho|T|$$

PROOF: For each $t \in T$, $\Pr[t \in S] = \rho$. By union bound, $\Pr[T \cap S = \emptyset] \geq 1 - \sum_{t \in S} \Pr[t \in S] = 1 - \rho|T|$.

On the other hand, let $\rho n = k$, we have,

$$\Pr[T \cap S = \emptyset] = (1 - \frac{k}{n}) \cdot (1 - \frac{k}{n-1}) \ldots (1 - \frac{k}{n - |T|}) \leq (1 - \frac{k}{n})^{|T|} \leq 1 - \rho|T|/2$$

$\square$

LEMMA 36
If $S$ is a uniformly random $\rho n$-sparse subset of $[n]$, then for any fix set $T$ with $\rho|T| \ll 1$ and $|T| \ll n$, $\rho|T| \geq \Pr[|S \cap T| = 1] \geq \rho|T|/2$

PROOF: For any $t \in T$, we compute $\Pr[S \cap T = \{t\}]$. Conditioned on $t \in T$, $S \setminus \{t\}$ is uniform $(\rho n - 1)$-sparse subset of $[n] \setminus \{t\}$, thus by the Lemma above,

$$\Pr[S \cap T = \{t\}] = \Pr[t \in S] \cdot \Pr[S \cap (T \setminus \{t\}) = \emptyset] \in [\rho/2, \rho]$$

All the events $S \cap T = \{t\}$ for $t \in T$ are disjoint, thus $\rho|T| \geq \Pr[|S \cap T| = 1] \geq \rho|T|/2$ $\square$

LEMMA 37
If $S$ is uniform random $\rho n$-sparse subset of $[n]$, then for any fix set $T$ with $\rho|T| \ll 1$ and $|T| \ll n$,

$$\Pr[|T \cap S| > \rho n/2] < (2\rho|T|)^{\rho n/2}$$

PROOF: Let $k = \rho n$. For any subsets $M$ of $T$ of size $k$, $\Pr[M \subset S] = \frac{k}{n} \cdot \frac{k}{n-1} \cdots \frac{k}{n-(k/2)} \leq (2\rho)^{k/2}$. Taking a union bound over all $M$, we have $\Pr[|T \cap S| \geq k/2] \leq \binom{|T|}{k/2}(2\rho)^{k/2} \leq (2\rho|T|)^{k/2}$. $\square$

**LEMMA 38**

*If $S$ is a uniformly random $\rho n$-sparse set, then for any two disjoint sets $T_1$ and $T_2$ with $\rho|T_1| \ll 1, \rho|T_2| \ll 1$ and $|T_1|, |T_2| \ll n$,*

$$\rho^2 |T_1||T_2|/5 \leq \Pr[S \cap T_1 \neq \emptyset, S \cap T_2 \neq \emptyset] \leq 2\rho^2 |T_1||T_2|$$

PROOF: From Lemma 36 we know $\Pr[|S \cap T_1| = 1] \in [\rho|T_1|/2, \rho|T_1|]$. Conditioned on $|S \cap T_1| = 1$, $S \setminus T_1$ is a random $\rho n - 1$ sparse vector in $n - |T_1|$ dimension. Apply Lemma 36 with $\rho' = (\rho n - 1)/(n - |T_1|)$ we have $\Pr[|S \cap T_2| = 1 \mid |S \cap T_1| = 1] = \Theta(\rho'|T_2|)$.

On the other hand, $\rho|T_1|$ is an upperbound on $\Pr[S \cap T_1 \neq \emptyset]$, $\rho'|T_2|$ is an upperbound on $\Pr[S \cap T_2 \neq \emptyset | S \cap T_1 \neq \emptyset]$, so we get the other direction. $\square$

**LEMMA 39**

*If $S$ is a uniformly random $\rho n$-sparse subset of $[n]$, then for two disjoint sets $T_1$, $T_2$ with $\rho|T_1|, \rho|T_2| \ll 1$ and $|T_1|, |T_2| \ll n$,*

$$\rho|T_1|/2 \leq \Pr[S \cap T_1 \neq \emptyset, S \cap T_2 = \emptyset] \leq 2\rho|T_1|$$

The proof is very similar to Lemma 38.

**LEMMA 40**

*If $S$ is a uniform random $\rho n$-sparse subset of $[n]$, then for any three sets $T_1$, $T_2$ and $T_3$ with $\rho|T_1|, \rho|T_2|, \rho|T_3| \ll 1$ and $|T_1|, |T_2|, |T_3| \ll n$. Let $\max\{|T_i|\} \leq A, \max_{1 \leq i < j \leq 3} |T_i \cap T_j| = B$, and $|T_1 \cap T_2 \cap T_3| = C$, then*

$$\Pr[(S \cap T_i \neq \emptyset) \text{ for } i = 1, 2, 3] \leq O(\rho^3 A^3 + \rho^2 AB + \rho C)$$

PROOF: The proof is not hard but tedious. It enumerates all the possible ways $S$ can intersect with all three $T_i$'s.

Let $K_1 = T_1 - T_2 \cup T_3$, $K_2 = T_2 - T_1 \cup T_3$, $K_3 = T_3 - T_1 \cup T_2$. Also, let $K_{12} = T_1 \cap T_2 - T_3$ (similarly define $K_{23}, K_{31}$). Finally let $K_{123} = T_1 \cap T_2 \cap T_3$.

Let $E$ be the event that $S \cap K_1 \neq \emptyset, S \cap K_2 \neq \emptyset, S \cap K_3 \neq \emptyset$; $E'$ be the event that $S \cap K_{123} \neq \emptyset$; $E_{12}$ be the event that $S \cap K_{12} \neq \emptyset, S \cap K_3 \neq \emptyset$ (similarly define $E_{23}$ and $E_{31}$). The event $S \cap T_i \neq \emptyset$ for all $i = 1, 2, 3$ is contained in the union of these events $E \cup E' \cup E_{12} \cup E_{23} \cup E_{31}$. By previous Lemmas it is easy to bound $\Pr[E] \leq 2\rho^3 |T_1||T_2||T_3| \leq 2\rho^3 A^3$, $\Pr[E'] \leq \rho C$, and $\Pr[E_{12}] \leq 2\rho^2 |K_{12}||K_3| \leq 2\rho^2 AB$. Hence $\Pr[(S \cap T_i \neq \emptyset) \text{ for } i = 1, 2, 3] \leq O(\rho^3 A^3 + \rho^2 AB + \rho C)$ $\square$

# L  Omitted Proofs in Section H

In this Section we shall prove Theorem 24 and Theorem 25.

**Proving Theorem 24** By the structure of the network we know

$$\mathbb{E}[y_u y_v y_s] \;=\; \sum_{i \in B(u), j \in B(v), k \in B(s)} w_{i,u} w_{j,v} w_{k,s} \, \mathbb{E}[h_i h_j h_k] \tag{3}$$

Observe that $w_{i,u}$ and $w_{j,v}$ are different random variables, no matter whether $i$ is equal to $j$. Thus by 3-order Hoeffding inequality stated in Lemma 47, with high probability over the choice of the weights

$$|\,\mathbb{E}[y_u y_v y_s]| \leq \sqrt{\sum_{i \in B(u), j \in B(v), k \in B(s)} \mathbb{E}[h_i h_j h_k]^2 \log^3 n} \tag{4}$$

Let $V$ be the main term in the bound:

$$V \triangleq \sum_{i \in B(u), j \in B(v), k \in B(s)} \mathbb{E}[h_i h_j h_k]^2.$$

Define $B_{uv} = B(u) \cap B(v)$, $B_{vs} = B(v) \cap B(s)$, $B_{su} = B(s) \cap B(u)$. Similar to Property $\mathcal{P}_{sing+}$, for a random graph we know all these sets have size at most 10.

The following claim gives bound on the variance term (so we know $\mathbb{E}[y_u y_v y_s]$ has small absolute value when they don't share a common neighbor):

**CLAIM 41**
*If $h$ is uniform random $\rho m$-sparse vector, and $B(u) \cap B(v) \cap B(s) = \emptyset$, and $|B_{uv}| \leq 10$, $|B_{vs}| \leq 10$ and $|B_{su}| \leq 10$, then*

$$V = \sum_{i \in B(u), j \in B(v), k \in B(s)} \mathbb{E}[h_i h_j h_k]^2 \leq O(\rho^6 d^3 + \rho^4 d)$$

**PROOF:** First of all, for any triple $(i, j, k)$ such that $i, j, k$ are distinct, the expectation $\mathbb{E}[h_i h_j h_k] \leq \rho^3$. The total number of such triples is at most $O(d^3)$. The contribution to the sum $V$ is $O(d^3 \rho^6)$.

Then we count the triples $(i, j, k)$ with $i = j$, and $i \neq k$ ($i \in B_{uv}$ and $k \in B(s)$). The total number of such triples $(i, j, k)$ is bounded by $10d$. In this case, $\mathbb{E}[h_i h_j h_k] \leq \rho^2$. Thus the total contribution to the $V$ is $O(d\rho^4)$. The other two symmetric cases have similar contribution to the $V$ as well.

Hence the total sum is bounded by $O(\rho^6 d^3 + \rho^4 d)$. $\square$

When $u, v, s$ share a common neighbor, the following claim shows $\mathbb{E}[y_u y_v y_s]$ has a large absolute value.

**CLAIM 42**
*If $h$ is uniform random $\rho m$-sparse vector , and $B(u) \cap B(v) \cap B(s) = \{z\}$, and $|B_{uv}| \leq 10$, $|B_{vs}| \leq 10$ and $|B_{su}| \leq 10$, then with high probability over the choice of the weights*

$$|\,\mathbb{E}[y_u y_v y_s]| \geq \rho - O(\sqrt{(\rho^6 d^3 + \rho^4 d) \log^3 n})$$

PROOF: Let $A_{z,z,z} = 0$, and $A_{i,j,k} = \mathbb{E}[h_i h_j h_k]$ when $(i,j,k) \neq (z,z,z)$. Thus

$$
\begin{aligned}
\mathbb{E}[y_u y_v y_s] &= \sum_{i \in B(u), j \in B(v), k \in B(s)} w_{i,u} w_{j,v} w_{k,s} \, \mathbb{E}[h_i h_j h_k] \\
&= w_{z,u} w_{z,v} w_{z,s} \, \mathbb{E}[h_z] + \sum_{i \in B(u), j \in B(v), k \in B(s)} w_{i,u} w_{j,v} w_{k,s} A_{i,j,k} \quad (5)
\end{aligned}
$$

Similar to the previous case, we know that

$$
\sum_{i \in B(u), j \in B(v), k \in B(s)} A_{i,j,k}^2 \leq O(\rho^6 d^3 + \rho^4 d)
$$

and thus with high probability,

$$
\begin{aligned}
\left| \sum_{i \in B(u), j \in B(v), k \in B(s)} w_{i,u} w_{j,v} w_{k,s} A_{i,j,k} \right| &\leq \sqrt{\sum_{i \in B(u), j \in B(v), k \in B(s)} A_{i,j,k}^2 \log^3 n} \\
&\leq O(\sqrt{(\rho^6 d^3 + \rho^4 d) \log^3 n})
\end{aligned}
$$

By Equation (5), the expectation can be bounded by

$$
\mathbb{E}[y_u y_v y_s] = w_{z,u} w_{z,v} w_{z,s} \, \mathbb{E}[h_z] \pm O(\sqrt{(\rho^6 d^3 + \rho^4 d) \log^3 n}),
$$

which implies

$$
| \, \mathbb{E}[y_u y_v y_s] | \geq \rho - O(\sqrt{(\rho^6 d^3 + \rho^4 d) \log^3 n})
$$

□

The Theorem follows directly from these two claims.

**Extending Theorem 24 to distribution $\mathcal{D}_1$**  Using similar ideas, and the bounds from Section E, we can prove Theorem 25.

We shall first extend Theorem 25 to the setting when $h$ is generated by the upper levels of the network (call the distribution $\mathcal{D}_1$). In order to do this we need to bound $\mathbb{E}[h_i h_j h_k]$ carefully. Following similar ideas as Lemma 12, we have

PROPOSITION 43
If $h \sim \mathcal{D}_1$, then

$$
\mathbb{E}[h_i h_j h_k] \text{ is } \begin{cases} \leq & O(\rho_1^3 + \rho_2 \rho_1 + \rho_2) \quad \text{if } B^+(i) \cap B^+(j) \cap B^+(k) \neq \emptyset, \\ \leq & O(\rho_1^3 + \rho_2 \rho_1 + \rho_3) \quad \text{if } B^+(i) \cap B^+(j) \cap B^+(k) = \emptyset, \end{cases}
$$

and

$$
\mathbb{E}[h_i h_j] \text{ is } \begin{cases} \leq & O(\rho_1^2 + \rho_2) \quad \text{if } B^+(i) \cap B^+(j) \neq \emptyset, \\ \leq & O(\rho_1^2 + \rho_3) \quad \text{if } B^+(i) \cap B^+(j) = \emptyset, \end{cases}
$$

Basically, most pairs/triples still have small expectation. The following Proposition shows there are only a very small number of pairs that are highly correlated.

PROPOSITION 44
*For a fixed set $T$ of size at most $3d$ in layer 1, when $d_1^5 < n$, with high probability, the number of pairs $i, j \in T, i \neq j$ such that $B^+(i) \cap B^+(j) \neq \emptyset$ is at most 10.*

PROOF: For any node $q$ in layer 2, let $C_q = \left| \left\{ (i, j) : i, j \in T, (q, i), (q, j) \in E_1^+ \right\} \right|$. By randomness of the graph we know $\mathbb{E}[C_q] \leq d^4/n^2$. Summing over all nodes in layer 2, we have $\mathbb{E}[\sum_q C_q] \leq d^4/n \leq n^{-0.2}$. With high probability, $\sum_q C_q \leq 10$. (Note that $C_q$ are independent variables.) □

With Proposition 43 and 44, we are ready to extend Theorem 24 to the more general distribution.

As before we need to bound the variance. When $B(u) \cap B(v) \cap B(s) = \emptyset$, we bound the variance $V = \sum_{i \in B(u), j \in B(v), k \in B(s)} \mathbb{E}[h_i h_j h_k]^2$ as follows:

First of all, there are at most $O(d^3)$ triples of $(i, j, k)$ such that $i, j, k$ are distinct. Typically, each one contribute to $V$ with $\mathbb{E}[h_i h_j h_k] \leq O((\rho_1^3 + \rho_2 \rho_1 + \rho_3)^2)$. However, there might be at most 10 triples with higher expectation (by Proposition 44), so the final bound is $O((\rho_1^3 + \rho_2 \rho_1 + \rho_2)^2)$. The total contribution of these triples is

$$v_1 = O(d^3(\rho_1^3 + \rho_2 \rho_1 + \rho_3)^2) + O((\rho_1^3 + \rho_2 \rho_1 + \rho_2)^2)$$

There are at most $10d$ triples of $(i, j, k)$ such that $i = j \in B_{uv}$ and $i \neq k$. The contribution of these triples is at most $\rho_1^2 + \rho_2$. Thus the total contribution is at most

$$v_2 \leq 10d(\rho_1^2 + \rho_2)^2).$$

Combining all these we know $\sqrt{V} \leq \sqrt{v_1 + v_2} \leq d^{3/2}(\rho_1^3 + \rho_2\rho_1 + \rho_3) + (\rho_1^3 + \rho_2\rho_1 + \rho_2) + 10d^{1/2}(\rho_1^2 + \rho_2)$.

Under the assumption of Theorem 25, $\sqrt{V \log^3 n} \leq \rho_1/3$. With high probability, when $u, v, s$ do not share a common neighbor, we know $|\mathbb{E}[y_u y_v y_s]| \leq \rho_1/3$ ; when $u, v, s$ share a unique neighbor $z$, we know $|\mathbb{E}[y_u y_v y_s]| \geq 2\rho_1/3$. Therefore the algorithm can distinguish them.

**Modifications in Graph Recovery** If edge weights are $\{\pm 1\}$, the same algorithm works because with high probability there are no more than a constant number of triples $u, v, s$ in the observed layer that share more than 1 neighbor (this is implied by Proposition 43).

When edge weights are in $[-1, 1]$, the hypergraph only satisfy the following approximate conditions: if $u, v, s$ share a common parent $z$ and the edges from $z$ to $u, v, s$ all have weight outside $[-0.01, 0.01]$, $(u, v, s)$ is an edge in the hypergraph; if $u, v, s$ do not share a common parent, then $(u, v, s)$ is not an edge. All other cases are ambiguous. Graph Recovery in this case can again be done by the same algorithm, because there are only a small constant fraction of ambiguous edges, and the thresholds in the algorithm tolerates a small constant fraction of error.

**Modifications in PartialEncoder**  Note that in this case graph recovery gives all edges, instead of positive edges, so we cannot directly use those for PartialEncoder. However, the expectation $\mathbb{E}[y_u y_v y_s]$ is positive and large, even if they share a common parent $z$ and there are even number of positive edges from $z$ to $u, v, s$. Using this information it is easy to separate the positive and negative edges, see Algorithm 9.

---

**Algorithm 9.**  LearnLastLayer

---

**Input:**  $E$, the set of edges

**Output:**  $E^+$, $E^-$, the set of positive and negative edges

  **for** each node $x \in h^{(1)}$ **do**

    Pick $u, v$ in $F(x)$

    Let $S$ be the set $\{s : \mathbb{E}_h[y_u y_v y_s] > 0\}$.

    **for** each $t : (x, t) \in E$ **do**

      **if** most triples $(t, u', v')(u', v' \in S)$ have positive $\mathbb{E}_h[y_t y_{u'} y_{v'}]$ **then**

        add $(x, t)$ to $E^+$

      **else**

        add $(x, t)$ to $E^-$

      **end if**

    **end for**

  **end for**

---

When weights are $\{+1, -1\}$ we are already done.

When the weights are in $[-1, 1]$ we don't find all the positive edges (1% of the edges are missing because they have small weights), but this is again OK because the PartialEncoder can tolerate small constant fraction of error.

**Modifications in Learning Decoder**  For learning decoder, since we do not have thresholds at the last layer, we only need to solve a system of linear equations.

# M  Standard Concentration Bounds

We use the following version of Hoeffding's inequality.

THEOREM 45
[Hoe63] $X_1, \ldots, X_n$ are independent random variables such that $X_i \in [a_i, b_i]$ almost surely. Let $S = X_1 + \cdots + X_n$. Then

$$\Pr[|S - \mathbb{E}[S]| \geq t] \leq 2 \exp(\frac{-2t^2}{\sum_{i=1}^n (a_i - b_i)^2})$$

Using this (and with union bound) it is easy to prove the following bilinear and trilinear bounds (these are not tight but we don't mind losing $\log n$ factors).

LEMMA 46
*If $w$, $w'$ are independent random vectors in $\{-1, 1\}^d$, then with high probability*

$$|\sum_{i,j} w_i w'_j A_{i,j}| \leq O\left(\sqrt{\sum_{i,j} A_{i,j}^2 \log^2 n}\right)$$

PROOF: We are going to bound $x^T A y$, which is equal to

$$x^T A y = \sum_{i=1}^{d} x_i \left(\sum_j A_{i,j} y_j\right)$$

With high probability over the choice of $y$, for any $i$, $\sum_j A_{i,j} y_j \leq O(\sqrt{\sum_j A_{i,j}^2 \log n})$. When these events happen, we have that

$$x^T A y = \sum_{i=1}^{d} x_i \left(\sum_j A_{i,j} y_j\right) \leq \sum_i x_i O(\sqrt{\sum_j A_{i,j}^2 \log n}).$$

By Hoeffding inequality again, we have that with high probability over the choice of $x$,

$$x^T A y \leq \sum_i x_i O\left(\sqrt{\sum_j A_{i,j}^2 \log n}\right) \leq O\left(\sqrt{\sum_{i,j} A_{i,j}^2 \log^2 n}\right)$$

$\square$

LEMMA 47
*If $w$, $w'$, $w''$ are independent random vectors in $\{-1, 1\}^d$, then with high probability*

$$|\sum_{i,j,k} w_i w'_j w''_k A_{i,j,k}| \leq O\left(\sqrt{\sum_{i,j,k} A_{i,j,k}^2 \log^3 n}\right)$$

The proof is very similar to Lemma 46.



\end{document}